\documentclass[accepted]{uai2026} 
                        
\usepackage{silence}
\usepackage[most]{tcolorbox}

\tcbset{
  highlight token/.style={
    boxrule=0.4pt,
    colback=gray!5,
    colframe=gray!50,
    fontupper=\ttfamily\color{purple},
    arc=1pt,
    boxsep=1pt,
    left=0.5pt,
    right=0.5pt,
    top=0.5pt,
    bottom=0.5pt,
    enhanced,
    box align=base,
    on line,
    height=1.2em,
    valign=center, 
  }
}

\usepackage[dvipsnames,table]{xcolor}
\definecolor{darkpink}{RGB}{199,21,140}

\definecolor{acadGreen}{RGB}{85, 140, 90}
\definecolor{softGray}{RGB}{210, 210, 210}

\definecolor{acadBlue}{RGB}{30, 70, 120}   

\definecolor{acadRed}{RGB}{140, 50, 50}    

\usepackage{adjustbox}
\usepackage{multirow}
\usepackage{booktabs, array}
\usepackage{makecell}
\usepackage{tabularx}
\usepackage{wrapfig}
\usepackage{makecell}
\usepackage{enumitem}

\usepackage{graphicx}
\usepackage[labelfont=bf]{caption}
\usepackage[format=hang]{subcaption}

\definecolor{citecolor}{RGB}{0,102,204}
\definecolor{linkcolor}{RGB}{190,105,30}
\definecolor{urlcolor}{RGB}{199,21,133}

\usepackage[acronym,nowarn,section,nogroupskip,nonumberlist]{glossaries}
\glsdisablehyper{}
\usepackage[nameinlink,capitalise]{cleveref}
\usepackage{aliascnt}
\creflabelformat{equation}{#2\textup{#1}#3}  
\crefname{section}{\S}{\S\S}
\usepackage{algorithm,algorithmicx,algpseudocode}

\usepackage{tikz}
\usepackage{forest}
\usetikzlibrary{decorations.pathreplacing, arrows.meta}

\newsavebox\CBox

\usepackage{amsthm}
\usepackage{dsfont}

\usepackage{amssymb}
\usepackage{pifont}

\newaliascnt{lemma}{theorem}
\newtheorem{lemma}[lemma]{Lemma}
\aliascntresetthe{lemma}
\crefname{lemma}{Lemma}{Lemmas}

\newaliascnt{definition}{theorem}
\newtheorem{definition}[definition]{Definition}
\aliascntresetthe{definition}
\crefname{definition}{Definition}{Definitions}

\newaliascnt{proposition}{theorem}

\aliascntresetthe{proposition}
\crefname{proposition}{Proposition}{Propositions}

\newaliascnt{corollary}{theorem}

\aliascntresetthe{corollary}
\crefname{corollary}{Corollary}{Corollaries}

\newaliascnt{remark}{theorem}

\aliascntresetthe{remark}
\crefname{remark}{Remark}{Remarks}

\newcounter{relctr} 
\everydisplay\expandafter{\the\everydisplay\setcounter{relctr}{0}} 

\newcommand\labelrel[2]{%
  \begingroup
    \refstepcounter{relctr}%
    \stackrel{\textnormal{(\alph{relctr})}}{\mathstrut{#1}}%
    \originallabel{#2}%
  \endgroup
}
\AtBeginDocument{\let\originallabel\label} 

\newenvironment{talign*}
 {\csname align*\endcsname}
 {\endalign}

 \usepackage{cancel}

\usepackage{mathtools,amsfonts,bm}

\newcommand{\eps}{\varepsilon}

\DeclareMathOperator*{\argmax}{arg\,max}

\usepackage{scalerel}
\newcommand{\smallth}[1]{^{\scaleto{(#1)}{5.5pt}}}
\newcommand{\nth}{\smallth{n}}

\newcommand{\Mod}[1]{\ (\mathrm{mod}\ #1)}
\renewcommand{\mod}[1]{\,\mathrm{mod}\,#1}
\newcommand{\aug}{\operatorname{aug}}

\newcommand{\ret}{R}
\newcommand{\deltainfty}{\delta_\infty}
\newcommand{\deltaone}{\delta_1}

\DeclareFontFamily{U}{cbgreek}{}
\DeclareFontShape{U}{cbgreek}{m}{n}{
        <-6>    grmn0500
        <6-7>   grmn0600
        <7-8>   grmn0700
        <8-9>   grmn0800
        <9-10>  grmn0900
        <10-12> grmn1000
        <12-17> grmn1200
        <17->   grmn1728
      }{}
\DeclareFontShape{U}{cbgreek}{bx}{n}{
        <-6>    grxn0500
        <6-7>   grxn0600
        <7-8>   grxn0700
        <8-9>   grxn0800
        <9-10>  grxn0900
        <10-12> grxn1000
        <12-17> grxn1200
        <17->   grxn1728
      }{}

\makeatletter
\newcommand{\normalorbold}{%
  \ifnum\pdf@strcmp{\math@version}{bold}=\z@ bx\else m\fi
}
\makeatother

\usepackage{thmtools}
\usepackage{thm-restate}

\usepackage{natbib} 
\usepackage[american]{babel}

\newcommand{\mytitle}{A Model-Free Universal AI}

\title{\mytitle}

\author[1]{\href{mailto:<yegonkim@kaist.ac.kr>?Subject=Your paper}{Yegon Kim}{}}
\author[1]{\href{mailto:<juholee@kaist.ac.kr>?Subject=Your paper}{Juho Lee}{}}
\affil[1]{
    Graduate School of AI\\
    KAIST\\
    Seoul, South Korea
}

\begin{document}
\maketitle

\begin{abstract}
In general reinforcement learning, all established optimal agents, including AIXI, are model-based, explicitly maintaining and using environment models.
This paper introduces Universal AI with Q-Induction (AIQI), the first model-free agent proven to be asymptotically $\varepsilon$-optimal in general RL.
AIQI performs universal induction over distributional action-value functions, instead of policies or environments like previous works.
Under a grain of truth condition, we prove that AIQI is strong asymptotically $\varepsilon$-optimal and asymptotically $\varepsilon$-Bayes-optimal.
We also apply our novel proof techniques to show asymptotic $\varepsilon$-optimality of Self-AIXI without any ad-hoc assumptions.
Our results significantly expand the diversity of known universal agents.
\end{abstract}

\section{Introduction}
\label{sec:intro}



The theory of general reinforcement learning \citep[GRL;][]{lattimore2014theory} provides a framework for studying agents in the most general class of environments, those that satisfy only minimal structural assumptions.
In this setting, {AIXI} \citep{hutter2005universal} is a foundational theoretical model: it combines universal induction \citep{solomonoff1964formal,solomonoff1964formal2} with sequential decision theory \citep{von1944theory,bellman1957dynamic} to define a Bayes-optimal agent.
Although AIXI is uncomputable, it serves as a useful theoretical model of powerful goal-driven AIs \citep{orseau2012memory,everitt2016self}, and has also been used to define a universal measure of intelligence, called the Legg-Hutter intelligence \citep{legg2007universal}.

Interestingly, AIXI and all other established optimal agents in GRL \citep{lattimore2014bayesian,leike2016thompson,cohen2019strongly,catt2023self}, otherwise known as universal agents or universal AI, are \emph{model-based}: they explicitly infer and make use of a model of the environment (a world model).
In contrast, the dominant paradigm in practical RL is \emph{model-free}, learning value functions or policies directly from experience \citep{watkins1989learning,williams1992simple,rummery1994line,konda1999actor}.
Showing that a model-free algorithm is optimal in GRL has therefore remained elusive and has been repeatedly highlighted as an open problem \citep{everitt2018universal,catt2022foundations,hutter2024introduction}.

In this paper, we present Universal \textbf{AI} with \textbf{Q}-\textbf{I}nduction (AIQI), a model-free agent that performs universal induction over \emph{return-predictors}, objects similar to distributional Q-values \citep{bellemare2017distributional}.
AIQI is essentially an $\eps$-greedy, on-policy distributional Monte Carlo control algorithm.
Under a grain of truth condition \citep{kalai1993rational},
we prove that AIQI achieves strong asymptotic $\eps$-optimality and asymptotic $\eps$-Bayes-optimality in GRL.
Our proof techniques can also be used to show asymptotic $\eps$-optimality of Self-AIXI without incurring ad-hoc assumptions such as in the proof by \citet{catt2023self}.
Our results significantly expand the class of known universal agents, 
and provide a blueprint for the analysis of other policy iteration algorithms in general environments.


\section{Background}

\subsection{General Reinforcement Learning}
\label{sec:grl}

We establish the basic notation for general reinforcement learning.
We also provide a summary of all the important symbols in \cref{sec:notation}.
Let $\mathcal{A}$, $\mathcal{O}$, and $\mathcal{R} \subseteq [0,1]$ denote the finite sets of actions, observations, and rewards, respectively. The percept space is $\mathcal{E} := \mathcal{O} \times \mathcal{R}$, and a history $h_{1:t} \in\mathcal{H} := (\mathcal{A} \times \mathcal{E})^\ast $ is a sequence of actions and percepts.
We write $h_{1:t} = a_1e_1 \dots a_{t}e_{t}$ for the history up to time $t$.
$\mathcal H_t := (\mathcal A \times \mathcal E)^{t}$ denotes the set of all histories up to time $t$.
A policy $\pi : \mathcal{H} \to \Delta\mathcal{A}$ maps histories to action distributions, while an environment $\nu : \mathcal{H} \times \mathcal{A} \to \Delta\mathcal{E}$ maps history-action pairs to percept distributions.
When policy $\pi$ interacts with environment $\nu$, they induce a distribution $\nu^\pi$ over $\mathcal H_t$:
\begin{align*}
\nu^\pi(h_{1:t})
&:=
{\pi(a_1)} \,
{ \nu(e_1  |  a_1)} \,
\cdots \,
{\pi(a_{t}  |  h_{<t})} \,
{ \nu(e_{t}  |  h_{<t} a_{t})}\\
&\phantom{:}= 
\textstyle
{\prod_{1\leq i \leq t} \pi(a_i  \mid  h_{<i})} \,
{\prod_{1\leq i \leq t} \nu(e_i  \mid  h_{<i} a_i)}
\\
&\phantom{:}=
{\pi(a_{1:t} \parallel e_{<t})} \,
{ \nu(e_{1:t} \parallel a_{1:t})}
\end{align*}
This is a valid distribution, that is, the probabilities sum to 1 across all histories $h_{1:t}$. A direct consequence is that for a deterministic policy $\pi'$,
\begin{equation}
\sum_{e_{1:t}\in\mathcal E^{t}} \nu \left(e_{1:t} \parallel \pi'\left(e_{<t}\right)\right) = 1,
\label{eq:nu-sum-one}
\end{equation}
where $\pi'(e_{<t})$ is the sequence of actions $a_{1:t}$ taken by $\pi'$ under fixed percepts $e_{<t}$.
We also define the sample space $\Omega:=(\mathcal{A}\times\mathcal{E})^\infty$, which consists of all \emph{infinite} action-percept sequences $h$. Every history $h_{1:t}$ is a prefix of some outcome $h$.

Given a discount factor $\gamma \in (0,1)$, the (full) \emph{return} at time $t$ is defined as
\[
\ret_t := (1-\gamma) \sum_{k=0}^\infty \gamma^k r_{t+k},
\]
where ${1-\gamma}=\left(\sum_{k=0}^\infty \gamma^k\right)^{-1}$ ensures that returns lie in $[0,1]$, like the rewards.
This definition assumes a geometric discount, and we discuss the generalization to general discount sequences in \cref{sec:arbitrary_discount}.

The \emph{value function} for policy $\pi$ in environment $\nu$ is
\[
V^\pi_\nu(h_{<t}) := \mathbb{E}^\pi_\nu \left[ \ret_t  \mid  h_{<t} \right],
\]
where $\mathbb{E}^\pi_\nu$ is the expectation with respect to $\nu^\pi$.
The optimal value is defined as $V^\ast_\nu = \sup_\pi V^\pi_\nu$, and an optimal policy $\pi^\ast _\nu$ is any policy that achieves this supremum.
The action-value, or Q-value function, is $Q^\pi_\nu(h_{<t}, a_t) := \mathbb{E}^\pi_\nu \left[ \ret_t  \mid  h_{<t} a_t \right]$.

Value functions satisfy the Bellman equations, which allow us to express one with the other:
\begin{align*}
V_\nu^\pi(h_{<t})
& \!=\! \mathbb{E}_{a_t\sim\pi(\cdot\mid h_{<t})}
   \left[Q_\nu^\pi(h_{<t},a_t)\right] \label{eq:bellman-V-exp}\\
Q_\nu^\pi(h_{<t},a_t)
& \!=\! \mathbb{E}_{e'_t\sim\nu(\cdot\mid h_{<t},a_t)}
\left[(1-\gamma) r_t + \gamma V_\nu^\pi(h'_{<t+1})\right]
\end{align*}

Obtaining the full return requires simulating or experiencing an infinite length trajectory. Given some horizon length $H \in \mathbb Z_+$, the $H$-step return is defined as
\[
\ret_{t,H} := (1-\gamma) \sum_{k=0}^{H-1} \gamma^k r_{t+k},
\]
which can be computed with $H$ steps of agent-environment interactions.
The $H$-step value function is defined as
$
V^{\pi}_{\nu,H}(h_{<t}) := \mathbb{E}^\pi_\nu \left[ \ret_{t,H} \mid  h_{<t} \right]
$.
Finally, we define the effective horizon as follows.
\begin{definition}[Effective horizon]
\label{def:effective_horizon}
For $\eta>0$, the $\eta$-effective horizon is 
\[
H(\eta):= \min \left\{ H \in \mathbb Z_+ \;\middle |\; (1-\gamma) \sum_{k=H}^\infty \gamma^k \leq \eta \right\}.
\]
\end{definition}
One can show that
the $H(\eta)$-step return differs from the full return by at most $\eta$, and likewise for the value function:
$
\ret_t - \ret_{t,H(\eta)}
= (1-\gamma) \sum_{k=H(\eta)}^\infty \gamma^k r_{t+k}
\leq \eta
$.

\subsection{AIXI}

To define AIXI, we first need to introduce the mixture environment $\xi$.
\begin{definition}[Mixture environment]
\label{def:mixture_environment}
For a countable environment class $\mathcal{M}$ and a prior distribution $w(\nu)>0$ over the environments $\nu$ in  $\mathcal M$, the mixture environment $\xi$ is defined as the unique environment that satisfies
\[
\xi(e_{1:t} \parallel a_{1:t}) = \sum_{\nu\in\mathcal M} w(\nu) \, \nu(e_{1:t} \parallel a_{1:t}).
\]
\end{definition}
Equivalently, $\xi^\pi(\cdot) = \sum_\nu w(\nu) \nu^\pi(\cdot)$ for any policy $\pi$.

AIXI is then defined as the optimal policy in the mixture environment $\xi$. By its definition, AIXI is the Bayes-optimal agent with respect to the prior $w(\nu)$.
\begin{definition}[AIXI]
\label{def:aixi}
AIXI is the optimal policy $\pi^\ast _\xi$ in the mixture environment $\xi$:
\[
\pi^\ast _\xi(h_{<t}) := \argmax_{a} Q^\ast _\xi(h_{<t}, a).
\]
\end{definition}

\section{Universal AI with Q-Induction}
\label{sec:aiqi}

We present Universal AI with Q-Induction (AIQI), a model-free approach to general reinforcement learning.

For a discretization level $M\in\mathbb Z_+$,
the $M$-discretized $H$-step return at time $t$ is defined as
\[
z_t := \frac{\left\lfloor M \ret_{t,H} \right\rfloor}{M},
\quad
z_t \in \mathcal Z := \left\{0,\frac{1}{M},\dotsc,\frac{M-1}{M}\right\}.
\]
Note that $z_t$ is an approximation of $\ret_{t,H}$ whose error gets smaller as $M\to\infty$. More precisely, it always holds that $\ret_{t,H}-z_t<1/M$.
We choose $z_t$ as the target of prediction instead of $\ret_{t,H}$.
Although this choice isn't necessary for our case of geometric discounting, there exist general discount sequences (\cref{sec:arbitrary_discount}) where $\ret_{t,H}$ can take an infinite number of possible values, which would complicate our analysis, and also the implementation.

We would like to be able to predict the discretized return $z_t$ conditioned on any history $h_{<t}$ and action $a_t$. One can naturally consider augmenting the history with returns,
\[
\dots
{\color{acadRed} a_{t-3}}
{\color{black} z_{t-3}}
{\color{acadBlue} e_{t-3}}
{\color{acadRed} a_{t-2}}
{\color{black} z_{t-2}}
{\color{acadBlue} e_{t-2}}
{\color{acadRed} a_{t-1}}
{\color{black} z_{t-1}}
{\color{acadBlue} e_{t-1}}
{\color{acadRed} a_t}
,
\]
which would be passed in to a Bayesian sequence predictor to predict $z_t$.
However, a complication arises:
computing the ground truth return $z_{t-1}$ requires the yet to be observed reward $r_t$ in $\color{acadBlue} e_t$.
Similarly, the returns $z_{t-H+1}, \dotsc, z_{t-1}$ are all unavailable at time $t$.
Therefore, instead of augmenting the sequence with returns at every step, we augment with a period $N\geq H$, at only the positions $i\equiv t \Mod N$:
\[
\dots
{\color{acadRed} a_{t-2N}}
{\color{black} z_{t-2N}}
{\color{acadBlue} e_{t-2N}}
\dots
{\color{acadRed} a_{t-N}}
{\color{black} z_{t-N}}
{\color{acadBlue} e_{t-N}}
\dots
{\color{acadRed} a_{t-1}}
{\color{acadBlue} e_{t-1}}
{\color{acadRed} a_t}
.
\]
Then, the augmented returns require rewards at times only up to $t-N+H-1 \leq t-1$, which are all available.
Why we don't simply consider $N=H$ will be revealed in the proof of \cref{thm:error-return-predictor}.
The periodically augmented sequence can be passed in to a Bayesian sequence predictor, to predict the next return $z_t$.
We formalize this below.

Given an augmentation \emph{period} $N\geq H$,
the augmented sample space associated with \emph{phase} $n \in \{0,\dotsc,N-1\}$, is
\[
\begin{aligned}
\tilde\Omega^{(n)}
\hspace{-2pt}
:=
\prod_{i=1}^\infty \mathcal B^{(n)}_i
\hspace{-2pt}
,
\hspace{7.2pt}
\mathcal B^{(n)}_i
\hspace{-2pt}
:=
\hspace{-2pt}
\begin{cases}
\mathcal A \times \mathcal Z \times \mathcal E, &\hspace{-7pt} i \equiv n \Mod N,\\
\mathcal A \times \mathcal E, &\hspace{-7pt} \text{otherwise}.
\end{cases}
\end{aligned}
\]
In other words, an augmented outcome $\tilde h\nth \in \tilde \Omega\nth$ contains some $\tilde z_{i_k}\in\mathcal Z$ at the periodic positions
$i_k=n+kN$, in between $a_{i_k}$ and $e_{i_k}$.
Note that these augmented returns do not necessarily have to match the ground truth returns $z_{i_k}$ computed with $r_{i_k:i_k{+}H{-}1}$, in order for $\tilde h\nth$ to qualify as an element of $\tilde \Omega\nth$.
We call an augmented outcome $\tilde h\nth$ ``valid'' if its augmented returns match the ground truth returns.
The prefix $\tilde h_{1:t}\nth$ is called an augmented history. The set of augmented histories is denoted $\tilde{\mathcal{H}}\nth $, and the set of augmented histories up to time $t$ is denoted $\tilde{\mathcal{H}}\nth _t$.

A \emph{return-predictor} with phase $n$ is any mapping $\phi: \bigcup_{i\equiv n} (\tilde{\mathcal{H}}\nth_{i-1}\times \mathcal A) \to \Delta\mathcal{Z}$. The domain consists exactly of all phase $n$ augmented histories that end exactly at positions where the augmented return $\tilde z_{i}$ should come next. We define the mixture return-predictor, like the mixture environment in \cref{def:mixture_environment}.
\begin{definition}[Mixture return-predictor]
\label{def:mixture-return-predictor}
Given a class $\mathcal P_n$ of phase $n$ return-predictors $\phi$, and a prior distribution $\omega_n (\phi)$, a mixture return-predictor $\psi_n$ is the unique phase $n$ return-predictor that satisfies
\begin{equation*}
\prod_{k=0}^K
\psi_n(\tilde z_{i_k} \mid \tilde h_{<i_k}\nth a_{i_k})
\hspace{-1pt}
=
\hspace{-2pt}
\sum_{\phi\in\mathcal P_n}
\omega_n(\phi)
\prod_{k=0}^K
\phi(\tilde z_{i_k} \mid \tilde h_{<i_k}\nth a_{i_k})
,
\end{equation*}
for all $K\geq 0$,
where $i_k = n + kN$.
\end{definition}
Equivalently, $\psi_n(\tilde z_{i_k}  \mid \tilde h_{<i_k}\nth a_{i_k})$ is given by the posterior predictive distribution
\[
\sum_{\phi \in \mathcal{P}_n} \omega_n(\phi \mid \tilde h_{<i_{k-1}}\nth a_{i_{k-1}}\tilde z_{i_{k-1}})
\,
\phi(\tilde z_{i_k} \mid \tilde h_{<i_k}\nth a_{i_k})
,
\]
with the posterior weights $\omega_n(\phi \mid \tilde h_{<i_{k}}\nth a_{i_{k}}\tilde z_{i_{k}})$ updated to
\[
\omega_n(\phi \mid \tilde h_{<i_{k-1}}\nth a_{i_{k-1}}\tilde z_{i_{k-1}}) \frac{\phi(\tilde z_{i_{k}} \mid \tilde h_{<i_{k}}\nth a_{i_{k}})}{\psi_n(\tilde z_{i_{k}} \mid \tilde h_{<i_{k}}\nth a_{i_{k}})},
\]
starting with $\omega_n(\phi \mid \epsilon) = \omega_n(\phi)$ where $\epsilon$ denotes the empty string.
In other words, by conditioning the mixture $\psi_n$ on a history, we are implicitly performing Bayesian inference over the return-predictors $\phi$.

Let $\aug_n : \Omega \hookrightarrow \tilde\Omega\nth$ be the injective mapping that augments an outcome $h$ to its corresponding valid augmented outcome $\aug_n (h)$, by inserting the true returns $z_{i_k}$ computed with the rewards $r_{i_k:i_k+H-1}$ in $h$.
Under a certain condition, a history $h_{<t}$ uniquely determines $\aug_n(h)_{<t}$ regardless of $h_{t:\infty}$, and in that case we abuse the notation $\aug_n$ to also denote the mapping $\aug_n(h_{<t})=\aug_n(h)_{<t}$.
Fortunately, we can state this condition explicitly as:
\begin{equation}
\label{eq:condition_aug}
(t-1-n) \mod N \geq H-1.
\end{equation}
This is because the latest index $i\equiv n \Mod N$ smaller than $t$ is exactly $t-1 - (t-1-n) \mod N$, and the return at this index requires rewards at times up to $t-1 - (t-1-n) \mod N + H -1 $, which is smaller than $t$ if and only if \cref{eq:condition_aug} holds.
In other words, all the rewards required for augmenting with the true returns $z_{i_k}$ are available in $h_{<t}$.
Note that \cref{eq:condition_aug} is trivially satisfied if $n=t \mod N$. We can hence define a \emph{unified predictor}, that
abstracts away the periodic augmentation.
\begin{definition}[Unified return-predictor]
Given a collection of $N$ mixture return-predictors $\{\psi_n\}_{0\leq n<N}$, the unified predictor $\psi: \mathcal H \times \mathcal A \to \Delta\mathcal{Z}$ is given by
\[
\psi(\tilde z_t  \mid  h_{<t} a_t)
:=
\psi_n (\tilde z_t  \mid  \aug_n (h_{<t}) a_t)
\]
where $n=t \mod N$.
\end{definition}

Next, we define AIQI as follows. We also provide a pseudocode in \cref{alg:aiqi}.
\begin{definition}[AIQI]
\label{def:aiqi}
Fix a return horizon $H$, a discretization level $M$, an augmentation period $N\geq H$, an exploration rate $\tau>0$, and a unified predictor $\psi$.
Let the Q-value estimate given history $h_{<t}$ be
\[
\hat Q(h_{<t},a_t) = \sum_{\tilde z_t\in\mathcal{Z}} \tilde z_t \cdot \psi (\tilde z_t \mid  h_{<t} a_t).
\]
AIQI, denoted $\hat \pi^{H,M,N,\tau}_\psi$ or $\hat \pi$, chooses the action that maximizes this estimate, with random exploration rate $\tau$:
\[
\hat\pi(a  \mid  h_{<t})
:= (1-\tau)\mathds 1 \left[a=a^\ast \right] + \tau/|\mathcal A|,
\]
where $a^\ast  = \argmax_{a_t} \hat Q(h_{<t},a_t)$.
We allow for stochastic tie-breaking when multiple actions maximize $\hat Q$.
\end{definition}
Note that AIQI doesn't simulate the environment or plan for the future.
Instead, AIQI directly predicts its own action-value, and at every step chooses the action that maximizes it.
As evidence accumulates, the value predictions become more and more accurate. In the limit, AIQI chooses the action that maximizes its true action-value, and will keep doing so in all future steps, thereby turning into a globally optimal policy.
Another way to look at AIQI is that it is performing policy evaluation at every step, treating the whole trajectory as having been generated by a single policy, and performing policy improvement in choosing the next action.

\begin{algorithm}[t]
\caption{Universal AI with Q-Induction (AIQI)}
\begin{algorithmic}[1]
\label{alg:aiqi}
\Require Return horizon $H$, discretization level $M$, augmentation period $N \ge H$, exploration probability $\tau$, set of unified predictor $\psi$
\State History $h_{<1} \gets \epsilon$, time $t \gets 1$
\Loop
    \For{each action $a \in \mathcal{A}$}
        \State $\hat{Q}(h_{<t}, a) \gets \sum_{\tilde{z} \in \mathcal{Z}} \tilde{z} \cdot \psi(\tilde{z} \mid h_{<t}, a)$
    \EndFor
    \State $a^\ast \gets \operatorname{argmax}_{a \in \mathcal{A}} \hat{Q}(h_{<t}, a)$
    \State Sample $p \sim \mathcal{U}[0,1]$
    \If{$p < 1 - \tau$}
        \State $a_t \gets a^\ast$ \Comment{Exploit}
    \Else
        \State $a_t \sim \mathcal{U}(\mathcal{A})$ \Comment{Explore}
    \EndIf
    \State Execute $a_t$, observe $e_t$
    \State $h_{<t+1} \gets h_{<t} a_t e_t$
    \State $t \gets t + 1$
\EndLoop
\end{algorithmic}
\end{algorithm}

As is common in the universal AI literature, AIQI faces the grain of truth (self-referential) problem \citep{kalai1993rational,leike2016formal,meulemans2025embedded,wyeth2025limit}.
Roughly speaking, every mixture return-predictor $\psi_n$ in $\psi$ has to contain in its class the true conditional return distribution under the AIQI policy $\hat \pi_\psi$, which itself depends on $\psi$.
The formal statement is given below.
\begin{definition}[Grain of truth]
\label{def:grain-of-truth}
Fix the AIQI parameters $H,M,N,\tau$.
A unified predictor $\psi$ has a grain of truth w.r.t. environment $\nu$ if: every $\psi_n$ is a mixture of a class $\mathcal P_n$ that contains the return-predictor
$\phi^\ast$ given by
\[
\phi^\ast (\tilde z_{i} \mid \tilde h_{<i}\nth a_{i})
= \nu^{\hat \pi} (z_{i} = \tilde z_i  \mid  h_{<i} a_{i}),
\enspace
i\equiv n \Mod N
,
\]
where $\hat \pi = \hat \pi^{H,M,N,\tau}_\psi$ and the right hand side is the conditional return distribution under $\nu^{\hat \pi}$.
\end{definition}
A nontrivial solution can be obtained by fixing a reflective oracle $O$ \citep{fallenstein2015reflective,leike2016formal,wyeth2025limit}, setting $\mathcal P_n$ to be the class of all $O$-estimable return-predictors, and assuming that $\nu$ is $O$-estimable.
Under this setup, the AIQI policy $\hat \pi$ is $O$-estimable, and $\phi^\ast = \nu^{\hat \pi}$ is also $O$-estimable, hence inside $\mathcal P_n$. Allowing for stochastic tie-breaking in \cref{def:aiqi} is essential for the use of reflective oracles.

We now introduce a notion of optimality that AIQI will be shown to satisfy. It is a more lenient form of strong asymptotic optimality, which is one of the strongest forms of optimality considered in the universal AI literature.
\begin{definition}[Strong Asymptotic $\eps$-Optimality]
Given $\eps>0$ and a class of environments $\mathcal M$, a policy $\pi$ is strong asymptotically $\eps$-optimal if, for all environments $\nu\in\mathcal M$, 
\[
\limsup_{t\to\infty} V_\nu^\ast (h_{<t}) - V_\nu^\pi (h_{<t}) \leq \eps,
\quad
\text{$\nu^\pi$-a.s.}
\]
\end{definition}
It has been shown by \citet[][Theorem~8]{lattimore2011asymptotically} that any deterministic policy, including AIXI, is not strong asymptotically $\eps$-optimal for all $\eps<1/4$.
\section{Results}
\label{sec:theoretical}

The aim of this section is to prove that:
\begin{itemize}
\item For arbitrarily small $\eps>0$, one can choose parameters $H,M,N,\tau,\psi$ with which AIQI satisfies \emph{strong asymptotic $\eps$-optimality}.
(\cref{thm:aiqi-eps-optimal})
\item With the same choice of parameters, AIQI is asymptotically $\eps$-optimal in the mixture environment $\xi$.
(\cref{thm:aiqi-approaches-aixi})
\item In contrast to AIXI, AIQI is not necessarily \emph{self-optimizing}---a property that roughly translates to ``off-policy asymptotic optimality''.
(\cref{thm:aiqi-not-self-optimizing})
\end{itemize}
Notably, we do not invoke any novel, ad-hoc assumptions in proving our results.
The list of notations can be found in \cref{sec:notation}, and omitted proofs are in \cref{sec:proofs}.
We have also performed experiments with a computable approximation of AIQI, which we report in \cref{sec:experimental}.
We also show in \cref{sec:self-aixi} that our novel proof techniques can be applied to the analysis of Self-AIXI \citep{catt2023self}, to get rid of strong assumptions such as off-policy sensibility.

\subsection{Convergence of Return-Predictor}

First, we show that the unified return-predictor $\psi$ converges to the true return-predictor $\phi^\ast = \nu^\pi$, when conditioned on increasingly longer histories generated by $\nu^\pi$.
A useful analytic device is the total variation (TV) distance.
Given two probability measures $P$ and $Q$ defined on a measurable space $(\Omega,\mathcal F)$, the TV distance is defined as
\[
D(P,Q)
:=
\sup_{E\in\mathcal F}\, |P(E)-Q(E)|
.
\]
An elementary property of TV distance is that, for a countable set $\{E_i\}$ of pairwise disjoint events,
\begin{equation}
\label{eq:partition-tv}
\textstyle
\sum_i |P(E_i)-Q(E_i)| \leq 2 \cdot D(P,Q).
\end{equation}

Given an environment $\nu$ and a policy $\pi$, define the phase $n$ augmented distribution $\tilde \nu^{\pi}_n := \nu^{\pi} \circ \aug_n^{-1} $ as the pushforward of $\nu^\pi$ by $\aug_n$.
We can express $\tilde \nu^{\pi}_n$ as a product of conditional probabilities whose dependencies respect the written order in
$\tilde h\nth=
{\color{acadRed} a_1}
{\color{acadBlue} e_1}
\dots
{\color{acadRed} a_n}
\tilde z_n
{\color{acadBlue} e_n}
\dots$, as follows:
\begin{align}
&\tilde \nu^{\pi}_n  (\tilde h \nth)
=
{\color{acadRed} \prod_{i} \tilde \nu^{\pi}_n(a_i  \mid  \tilde h\nth_{<i})}
{\color{acadBlue} \prod_{i\not\equiv{n}} \tilde \nu^{\pi}_n(e_i  \mid  \tilde h\nth_{<i} a_i )}
\nonumber
\\
&\hspace{45pt}
\cdot
{\color{acadBlue} \prod_{i\equiv{n}} \tilde \nu^{\pi}_n(e_i  \mid  \tilde h\nth_{<i} a_i \tilde z_i)}
\prod_{i\equiv n} \tilde \nu^{\pi}_n (\tilde z_i \mid \tilde{h}\nth_{<i} a_i)
.
\label{eq:conditional-tilde-nu}
\end{align}
Note that the above characterization of $\tilde \nu^\pi_n$ is informal, as $\tilde \nu^{\pi}_n  (\tilde h \nth)$ is just 0 for most $\tilde h \nth$. A proper characterization would involve $\tilde \nu^{\pi}_n  (\tilde h \nth_{<t})$ and all products indexed up to $i=t-1$. We use the informal notation for brevity.

The last factor in \cref{eq:conditional-tilde-nu} is simply
$\nu^{\pi}(z_i=\tilde z_i \mid {h}_{<i} a_i)$,
which we abbreviate as
$
\nu^{\pi}(\tilde z_i \mid {h}_{<i} a_i)
$.
Let us also abbreviate the first three factors in \cref{eq:conditional-tilde-nu} as $\tilde \nu^{\pi}_n (h \parallel \tilde z_{i\equiv n})$, so that
\begin{equation}
\label{eq:tilde-nu-def}
\tilde \nu^{\pi}_n  (\tilde h \nth)
=
\tilde \nu^{\pi}_n  (h  \parallel \tilde z_{i\equiv n})
\prod_{i\equiv n} \nu^{\pi}(\tilde z_i \mid {h}_{<i} a_i).
\end{equation}

Let $\pi$ be the AIQI policy $\hat \pi$ with some parameters $H$,$M$,$N$,$\tau$, $\psi$.
Suppose its mixture return-predictor $\psi_n$ in $\psi$ is a mixture of return-predictors including $\phi^\ast$ given by
\[
\phi^\ast (\tilde z_{i} \mid \tilde h_{<i}\nth a_{i})
= \nu^{\pi} (\tilde z_i  \mid  h_{<i} a_{i}),
\quad
i\equiv n \Mod N
,
\]
which is exactly the grain of truth condition. Then, define
\[
q_n (\tilde h\nth)
:= \tilde \nu^\pi_n (h  \parallel  \tilde z_{i\equiv n})  \prod_{i\equiv n} \psi_n (\tilde z_i  \mid  \tilde h_{<i}\nth a_i),
\]
which is \cref{eq:tilde-nu-def} with $\nu^\pi=\phi^\ast$ replaced by $\psi_n$.
Intuitively, we want to show that our mixture $\psi_n$ converges to $\phi^\ast$, and we do so by first showing that $q_n$ converges to $\tilde \nu^\pi_n$.

By the definition of a mixture in \cref{def:mixture-return-predictor},
\begin{align*}
q_n (\tilde h\nth)
&\geq
\tilde \nu^\pi_n (h  \parallel  \tilde z_{i\equiv n})
 \left( \omega_n (\phi^\ast)  \prod_{i\equiv n} \phi^\ast (\tilde z_i  \mid  \tilde h\nth_{<i} a_i) \right)\\
&= \omega_n (\phi^\ast) \tilde \nu^\pi_n (\tilde h\nth)
.
\end{align*}
By the inequality, $\tilde \nu^\pi_n$ is absolutely continuous with respect to $q_n$, so that by \citet{blackwell1962merging},
\[
\lim_{t\to\infty} D(q_n, \tilde \nu^\pi _n  \mid  \tilde h\nth_{<t}) = 0,
\quad
\tilde \nu^\pi_n \text{-a.s.}
\]

We can then prove the following lemma.
Intuitively, $\tilde \nu^\pi_n$ is the pushforward of $\nu^\pi$ by $\aug_n$, so we can replace $\tilde h\nth_{<t}$ with ${{\aug_n}(h)}_{<t}$, and $\tilde \nu^\pi_n$- with $\nu^\pi$-almost sure convergence.
\begin{restatable}[Convergence in TV distance]{lemma}{convergencetv}
\label{thm:convergence_tv}
Let $\pi$ be the AIQI policy $\hat \pi^{H,M,N,\tau}_\psi$ where $\psi$ has a grain of truth w.r.t. an environment $\nu$.
There exists a $\nu^\pi$-probability-one set $S \subseteq \Omega$ such that for all $h\in S$ and $n\in [0,N-1]$,
\[
\lim_{t\to\infty} D(q_n, \tilde \nu^\pi _n  \mid  {{\aug_n}(h)}_{<t}) = 0.
\]
\end{restatable}

We are now ready to show that each mixture return-predictor $\psi_n$, or more generally the unified predictor $\psi$, converges to the true return-predictor $\phi^\ast = \nu^\pi$.
By ``convergence'', one would usually mean something along the line of
\begin{equation}
\lim_{t\to\infty}
\sum_{\tilde z_t \in \mathcal Z}
\left|
\psi (\tilde z_{t}  |  h_{<t}\, a_{t})
-
\nu^\pi (\tilde z_{t}  |  h_{<t}\, a_{t})
\right|
= 0.
\label{eq:return-predictor-error-usual}
\end{equation}
Instead, we prove something strictly stronger, involving a sum over hypothetical future trajectories $h'_{t:m-1}$ that extend $h_{<t}$.
This stronger form of convergence is essential for proving \cref{thm:aiqi-eps-optimal} without invoking any ad-hoc assumptions, e.g., off-policy sensibility \citep[Theorem~16]{catt2023self}.

Let us denote the return-predictor error by
\[
\delta_\psi (h_{<t} a_t)
:=
\sum_{\tilde z_t\in \mathcal Z}
\left|
\psi (\tilde z_t  \mid  h_{<t} \, a_{t})
-
\nu^\pi (\tilde z_t  \mid h_{<t} \, a_{t})
\right|,
\]
and let $\mathcal T$ be the set of hypothetical trajectories $h'_{<m}$,
\begin{align}
\label{eq:T}
\mathcal T := \{h'_{<m} & \mid  h'_{<t}=h_{<t},\, e'_{t:m-1} \in \mathcal E^{m{-}t},
\nonumber
\\
&\hspace{24pt}
a'_{t:m-1}= \pi' (h_{<t}, e'_{t:m-2}) \},
\end{align}
formed by $h_{<t}$ appended with all possible percept sequences of length $m-t$ and reactions from some \emph{deterministic policy} $\pi'$.
We assume for now that $\pi'$ is an arbitrary deterministic policy, and specify it later in the proof of \cref{thm:aiqi-eps-optimal}.
We can show the following result.
\begin{restatable}[Convergence of return-predictor]{lemma}{errorreturnpredictor}
\label{thm:error-return-predictor}
Let $\pi$ be the AIQI policy $\hat \pi^{H,M,N,\tau}_\psi$ where $\psi$ has a grain of truth w.r.t. an environment $\nu$.
For every $\beta>0$, there exists a $\nu^\pi$-probability-one set $S \subseteq \Omega$ where, for every outcome $h\in S$, there exists $t_0$ such that: for $t\geq t_0$ and $m\in[t,t{+}N{-}H]$,
\begin{talign*}
\sum_{\substack{h'_{<m} \in \mathcal T, a'_m\in\mathcal A}}
\nu (e'_{t:m-1}  \mid  h_{<t} \parallel a'_{t:m-1})
\cdot
\delta_\psi (h'_{<m} a'_m)
&
\\
< 2\beta \left( \tau/|\mathcal A| \right)^{-(m-t+1)}
&.
\end{talign*}
\end{restatable}
A corollary of our result is \cref{eq:return-predictor-error-usual}, obtained by setting $m=t$.
Remarkably, our result is stronger, bounding the sum over hypothetical trajectories for all $m\in[t,t+N-H]$.
This is possible due to the subtle fact that ${\aug_n}(h_{<t})$ is well-defined for $n=m \mod N$, according to \cref{eq:condition_aug}.
This subtle fact also reveals why we can't simply set $N=H$ and have to add a buffer of length $N-H$. In fact, the length of this buffer is exactly the maximum length of hypothetical future trajectories that can extend $h_{<t}$ in our bound.

\subsection{One-Step Optimality}

We continue to denote the AIQI policy $\hat\pi$ as $\pi$.
Let the AIQI parameter $H$ be the $\eta$-effective horizon $H(\eta)$ for some $\eta>0$.
We can show that the estimated action-value $\hat Q$ in \cref{def:aiqi} converges to the true action-value $Q^\pi_\nu$.
Denote the Q-value estimation error by
\[
\delta_Q (h_{<t} a_t)
:=
|
\hat Q (h_{<t} a_{t})
-
Q_\nu^\pi (h_{<t} a_{t})
|.
\]
\begin{restatable}[Convergence of Q-value prediction]{lemma}{errorqestimation}
\label{thm:error-q-estimation}
Under the conditions established in \cref{thm:error-return-predictor},
\begin{talign*}
\sum_{\substack{h'_{<m}\in \mathcal{T}, a'_m\in\mathcal A}}
\nu (e'_{t:m-1}  \mid  h_{<t} \parallel a'_{t:m-1})
\cdot
\delta_Q (h'_{<m} a'_m)
&
\\
<
2\beta \left( \tau/{|\mathcal A|} \right)^{-(m-t+1)} + M^{-1} + \eta
&
.
\end{talign*}
\end{restatable}

Next, we show that the difference between the on-policy value and maximum Q-value becomes small.
We call this difference the one-step optimality gap,
\[
\deltaone (h_{<t}) := \max_a Q^\pi_\nu(h_{<t}, a) - V^\pi_\nu(h_{<t})
,
\]
which is non-negative due to $V(\cdot)=\mathbb E_a [Q(\cdot,a)]$.
\begin{restatable}[Convergence of one-step optimality gap]{lemma}{erroronestep}
\label{thm:error-one-step}
Under the conditions established in \cref{thm:error-return-predictor},
\begin{talign*}
\sum_{h'_{<m} \in \mathcal{T}} \nu(e'_{t:m-1}  \mid  h_{<t} \parallel a'_{t:m-1})
\cdot
\deltaone (h'_{<m})
\hspace{40pt}
&
\\
< 
2\cdot (2\beta \left( \tau/{|\mathcal A|} \right)^{-(m-t+1)} + M^{-1} + \eta) + 2\tau
&
.
\end{talign*}
\end{restatable}
Intuitively, since AIQI chooses the action $a^\ast$ that maximizes the estimate $\hat Q$, the value $V$ of AIQI is approximately equal to its Q-value under the action $a^\ast$, with only a slight difference due to the exploration probability $\tau$. 
As the estimate $\hat Q$ converges to the true Q-value due to \cref{thm:error-q-estimation}, the Q-value $Q(\cdot, a^\ast)$, and hence the value $V(\cdot)$, converges to the true maximum $\max_a Q(\cdot,a)$.

\subsection{Asymptotic Optimality}
\label{sec:one-step-to-global}
Let the global optimality gap be
\[
\deltainfty(h_{<t}) := V^\ast _\nu(h_{<t}) - V^\pi_\nu(h_{<t})
.
\]
We now want to show the convergence of the global optimality gap, which is equivalent to asymptotic optimality.
First, we present a lemma that relates the global optimality gap to the one-step optimality gap.
\begin{restatable}[Bound on global optimality gap]{lemma}{onestepandglobal}
\label{thm:one-step-and-global}
For any environment $\nu$, policy $\pi$, and history $h_{<t}$,
\[
\deltainfty(h_{<t}) \leq \gamma \sum_{e'_t} \nu(e'_t  \mid  h_{<t} \, a'_t)\, \deltainfty(h_{<t} \,a'_t\, e'_t) + \deltaone(h_{<t}),
\]
where
$a'_t = \argmax_a \left( Q^\ast_\nu (h_{<t},a) - Q^\pi_\nu (h_{<t},a) \right) $.
\end{restatable}

Note that we can form a \emph{chain} of inequalities by repeatedly applying the lemma to $\deltainfty$ on the right hand side, as follows:
\begin{align*}
&\deltainfty(h_{<t})
\leq
\gamma \sum_{e'_t} \nu(e'_t  \mid  h_{<t} a'_t) \, { \deltainfty} (h'_{<t+1}) + { \deltaone}(h_{<t})
\nonumber
\\
&\hspace{14pt}
\leq
\gamma^2 \sum_{e'_{t:t+1}} \nu(e'_{t:t+1}  \mid  h_{<t} \parallel a'_{t:t+1}) \, { \deltainfty}(h'_{<t+2})
\nonumber
\\
&\hspace{20pt}
+ \gamma \sum_{e'_t} \nu(e'_t \mid h_{<t} a'_t) \, { \deltaone} (h'_{<t+1}) 
+ { \deltaone} (h_{<t})
\nonumber
\\
&\hspace{14pt}
\leq \dots
\nonumber
\\
&\hspace{14pt}
\leq
\gamma^L
\hspace{-6pt}
\sum_{e'_{t:t+L-1}}
\hspace{-6pt}
\nu(e'_{t:t+L-1}  \mid  h_{<t} \parallel a'_{t:t+L-1}) \, { \deltainfty}(h'_{<t+L})
\nonumber
\\
&\hspace{12.5pt}
+
\sum_{l=0}^{L-1} \gamma^l
\big[
\hspace{-6pt}
\sum_{e'_{t:t+l-1}}
\hspace{-6pt}
\nu(e'_{t:t+l-1}  \mid  h_{<t} \parallel a'_{t:t+l-1}) \, { \deltaone}(h'_{<t+l})
\big]
\nonumber
\end{align*}
where
$
a'_{i} = \argmax_a | Q^\ast_\nu (h'_{<i},a) - Q^\pi_\nu (h'_{<i},a) |
$.

Furthermore, the first term in the last inequality is 
smaller than or equal to $\gamma^L$, since $\delta_\infty(h'_{<t+L}) \leq 1$.
Thus $\deltainfty(h_{<t})$ is smaller than or equal to
\begin{equation}
\gamma^L
\!+\!
\sum_{l=0}^{L-1} \gamma^l
\hspace{-4pt}
\sum_{e'_{t:t+l-1}}
\hspace{-6pt}
\nu(e'_{t:t+l-1}  |  h_{<t} \| a'_{t:t+l-1})\,
\deltaone(h'_{<t+l}).
\label{eq:chain-global-one-step}
\end{equation}
We have bound the global optimality gap $\deltainfty (h_{<t})$ using \emph{only} one-step optimality gaps $\deltaone$.

We are now ready to prove the central result of our paper, that AIQI is strong asymptotically $\eps$-optimal.
\begin{restatable}[AIQI is strong asymptotically $\eps$-optimal]{theorem}{aiqiepsoptimal}
\label{thm:aiqi-eps-optimal}
Fix an environment class $\mathcal M$ and tolerance $\eps>0$.
Let
$\tau \leq \frac{\eps (1-\gamma)}{10}$,
$M \geq \frac{10}{\eps (1-\gamma)} $,
$H=H(\eta)$ with $\eta \leq \frac{\eps (1-\gamma)}{10}$, and
$N \geq H+ \log_\gamma \frac{\eps}{5}$.
Suppose $\psi$ is a unified predictor that has a grain of truth with respect to all environments $\nu \in \mathcal M$ and the above parameters for AIQI.
Then, the AIQI policy $\hat \pi^{H,M,N,\tau}_\psi$ is strong asymptotically $\eps$-optimal in $\mathcal M$.
\end{restatable}
\begin{proof}
We simply write $\pi$ for the AIQI policy $\hat \pi^{H,M,N,\tau}_\psi$.
Fix an environment $\nu \in \mathcal M$, and let $L = N-H+1$.
Fix some $\beta>0$, and with it, choose $S\subseteq \Omega$ as in \cref{thm:error-return-predictor}. For any $h\in S$, choose $t_0$, again as in \cref{thm:error-return-predictor}. Then for all $t\geq t_0$, substituting $m=t+l$ in \cref{eq:chain-global-one-step} and noting that
\[
a'_{i} = \argmax_a | Q^\ast_\nu (h'_{<i},a) - Q^\pi_\nu (h'_{<i},a) |
\]
is the output of a \emph{deterministic policy}, we apply \cref{thm:error-one-step} to obtain
\begin{align*}
\deltainfty (h_{<t})
<
\gamma^L
+
\sum_{l=0}^{L-1} \gamma^l
\bigg[
2 (2\beta \left( \frac{\tau}{|\mathcal A|} \right)^{-(l+1)}&
\\
+ M^{-1} + \eta)
+ 2\tau&
\bigg]
.
\end{align*}
Recall that
$L\geq \log_\gamma \frac{\eps}{5}$,
$\tau \leq \frac{\eps (1-\gamma)}{10}$,
$M\geq \frac{10}{\eps (1-\gamma)}$,
and
$\eta \leq \frac{\eps (1-\gamma)}{10}$.
If we use
$
\beta = \frac{\eps \tau}{20} |\mathcal A|^{-1} ( \sum_{l=0}^{L-1} (\gamma |\mathcal A| /\tau )^l )^{-1}
$ in choosing $S$ and $t_0$,
the right hand side is smaller than or equal to $\eps$.
Thus for $h\in S$ and $t\geq t_0$, $\deltainfty (h_{<t}) \leq \eps$.
This concludes our proof that
\[
\limsup_{t\to\infty} \deltainfty (h_{<t})
\leq \eps,
\quad
\text{$\nu^\pi$-a.s.}
\qedhere
\]
\end{proof}

We can also show that strong asymptotic $\eps$-optimality in the environment class $\mathcal M$ implies asymptotic optimality in the mixture environment $\xi$.
\begin{restatable}[Optimality in $\mathcal M$ implies optimality in $\xi$]{lemma}{strongimpliesbayes}
\label{thm:strong-implies-bayes}
Given a mixture $\xi$ of environment class $\mathcal M$,
and a policy $\pi$ that is strong asymptotically $\eps$-optimal with respect to $\mathcal M$,
the same policy $\pi$ is asymptotically $\eps$-optimal in $\xi$:
\[
\limsup_{t\to\infty} V_\xi^\ast (h_{<t}) - V_\xi^\pi (h_{<t}) \leq \eps
\]
holds both $\xi^\pi$-a.s. and, for all $\nu\in\mathcal M$, $\nu^\pi$-a.s.
\end{restatable}

As a consequence, AIQI is also asymptotically optimal in $\xi$.
In other words, AIQI approximates \emph{Bayes-optimality}, which is the defining feature of AIXI.
\begin{restatable}[AIQI is asymptotically $\eps$-Bayes-optimal]{theorem}{aiqiapproachesaixi}
\label{thm:aiqi-approaches-aixi}
For any mixture $\xi$ of environment class $\mathcal M$,
the AIQI policy $\hat \pi^{H,M,N,\tau}_\psi$ with the parameters in \cref{thm:aiqi-eps-optimal} is asymptotically $\eps$-optimal in $\xi$.
\end{restatable}
\begin{proof}
\cref{thm:aiqi-eps-optimal} and \cref{thm:strong-implies-bayes}.
\end{proof}

The single-agent result extends directly to multi-agent interaction by viewing each agent $i$ as acting in the subjective environment $\sigma_i$ induced by the other agents' policies $\pi_{\neq i}$ \citep[Section 4.1]{leike2016formal}. Thus, if each AIQI agent $\pi_i=\hat\pi^{H_i,M_i,N_i,\tau_i}_{\psi_i}$ satisfies the assumptions of \cref{thm:aiqi-eps-optimal}, every agent is eventually an $\eps$-best response to the others. Hence the joint policy profile $\pi_{1:n}$ converges to an $\eps$-Nash equilibrium, $\sigma^{\pi_{1:n}}$-almost surely. We detail this extension in \cref{sec:multi-agent-environment}.

\subsection{Off-Policy Behavior}
\label{sec:off-policy}

Until now we have solely discussed the on-policy setting, where the policy that interacts with the environment is the same as the policy that we are optimizing.
In the \emph{off-policy setting}, we must infer an optimal policy from the interactions generated by some ``historic policy'' $\pi'$ that
we {do not} have control over.
We can formalize this notion of optimality with the \emph{self-optimizing} property.
\begin{definition}[Self-optimizing policy]
\label{def:self-optimizing}
A policy $\color{acadRed} \bar{\pi}$ is called self-optimizing for a class of environments $\mathcal{M}$ and a historic policy $\color{acadBlue} \pi'$, if for every $\nu \in \mathcal M$,
\[
\lim_{t\to\infty} V_\nu^\ast(h_{<t}) - V_\nu^{{\color{acadRed} \bar \pi}}(h_{<t}) = 0,
\quad
\nu^{\color{acadBlue} \pi'}\text{-a.s.}
\]
\end{definition}
\citet{hutter2002self} showed a remarkable property of AIXI, that if there \emph{exists} a self-optimizing policy for $\mathcal M$ and $\pi'$, then AIXI is one of those self-optimizing policies.
On the other hand, we can expect that AIQI does not satisfy this property, since it is an on-policy MC control algorithm.

In fact, if the predictor $\psi$ supports the return-predictor $\phi^\ast = \nu^{\pi'}$, AIQI will eventually choose actions with the highest value under the historic policy $\pi'$, and not its own value---the optimality of AIQI is thus no longer guaranteed. Even if we relax the self-optimizing property to the $\eps$-self-optimizing property by replacing $\lim =0$ with $\limsup \leq \eps$, the following can be shown.
\begin{restatable}[AIQI is not self-optimizing]{theorem}{aiqinotselfoptimizing}
\label{thm:aiqi-not-self-optimizing}
Suppose the unified predictor $\psi$ is built from classes $\mathcal P_n$ that contain all computable variable-order Markov models.
For any period $N$, exploration rate $\tau$, horizon $H\geq 2$, tolerance $\eps$ small enough (depending on $H$), and discretization level $M$ large enough
(depending on $H,\eps,\tau$), there exist
an environment class $\mathcal M$, a historic policy $\pi'$, and a discount factor
$\gamma\in(0,1)$, for which there exists a self-optimizing policy, but the AIQI policy $\hat\pi^{H,M,N,\tau}_\psi$ is not even $\eps$-self-optimizing.
\end{restatable}

This completes our investigation of the theoretical properties of AIQI.
We have assumed a geometric discount setting for simplicity,
but we can extend all our results to the general discount setting, where returns are defined as $R_t = \sum_{k=t}^\infty \gamma_k r_k / \sum_{k=t}^\infty \gamma_k$ for a general discount sequence $\gamma_k$.
Specifically, our results can be generalized to cases where the discount sequence decays faster than a geometric sequence.
We detail the generalization in \cref{sec:arbitrary_discount}.

\section{Application to Self-AIXI}
\label{sec:self-aixi}

Self-AIXI \citep{catt2023self} is similar to AIQI in that it makes predictions of its own action values, and chooses the action $a_t$ that maximizes the predicted value.
As noted by \citet{wyeth2025unbounded},
\citet{catt2023self} make unjustified assumptions in their attempt to prove the asymptotic optimality of Self-AIXI.
We suggest adding $\eps$-greedy exploration to Self-AIXI and proving its asymptotic $\eps$-optimality without any ad-hoc assumptions.
Fortunately, our proof techniques can be applied with minimal changes.
The key is in proving an analogue of \cref{thm:error-q-estimation}.
We present the key results below, and provide their proofs in \cref{sec:proofs-self-aixi}.
First, we define the $\eps$-greedy Self-AIXI.

\begin{definition}[$\eps$-greedy Self-AIXI]
\label{def:self-aixi}
Let $\xi$ be a mixture environment constructed from some environment class $\mathcal M$, and $\zeta$ be a mixture policy constructed from some policy class $\mathcal P$.
Then, $\eps$-greedy Self-AIXI, denoted $\pi_S$, is defined as
\[
\pi_S (a \mid h_{<t})
:= (1-\tau)\mathds 1 \left[a=a^\ast \right] + \tau/|\mathcal A|,
\]
where $a^\ast  = \argmax_{a_t} Q^\zeta_\xi (h_{<t}, a_t)$.
We allow for stochastic tie-breaking when multiple actions maximize $Q^\zeta_\xi$.
\end{definition}

The following lemmas serve as a bridge between Self-AIXI and our proof technique.

\begin{restatable}[Convergence of $\xi^\zeta$ to $\nu^\pi$]{lemma}{selfaixitvconvergence}
\label{thm:self-aixi-tv-convergence}
For any environment $\nu\in\mathcal M$ and policy $\pi \in \mathcal \mathcal P$,
\[
\lim_{t\to\infty} D\!\left(\xi^{\zeta}, \nu^{\pi} \mid h_{<t}\right) = 0,
\quad \nu^{\pi}\text{-a.s.}
\]
\end{restatable}

\begin{lemma}[Bounds on Q-value difference]
\label{thm:self-aixi-q-value-difference-bound}
For any two policies $\pi_1, \pi_2$ and two environments $\nu_1, \nu_2$,
\[
\left| Q_{\nu_1}^{\pi_1}(h_{<t},a_t) - Q_{\nu_2}^{\pi_2}(h_{<t},a_t) \right|
\leq
D\!\left(\nu_1^{\pi_1}, \nu_2^{\pi_2} \mid h_{<t}\, a_t\right)
\]
\end{lemma}
\begin{proof}
See \citet[Lemma 4.17]{leike2016nonparametric}.
\end{proof}

\begin{restatable}[Average conditional TV bound]{lemma}{avgcondtv}
\label{thm:avg_cond_tv}
Let \(P\) and \(Q\) be probability measures on \(\mathcal X \times \mathcal Y\), where \(\mathcal X\) is finite. Let \(P_X\) and \(Q_X\) denote the marginals on \(\mathcal X\). Then
\[
\sum_{x \in \mathcal X} P_X(x)\,
D \bigl(P, Q \mid X=x \bigr)
\;\le\;
2D(P,Q).
\]
\end{restatable}

Now we can prove an analogue of \cref{thm:error-q-estimation} for $\pi_S$. Let
\[
\delta_Q (h_{<t} a_t)
:=
|
Q^\zeta_\xi (h_{<t} a_{t})
-
Q_\nu^\pi (h_{<t} a_{t})
|,
\]
and define $\mathcal T$ as in \cref{eq:T}.

\begin{restatable}[Convergence of Q-value prediction for $\eps$-greedy Self-AIXI]{lemma}{exploringselfaixierrorqestimation}
\label{thm:exploring-self-aixi-error-q-estimation}
Let $\pi = \pi_S$.
For every $\beta>0$, there exists a $\nu^\pi$-probability-one set $S \subseteq \Omega$ where, for every outcome $h\in S$, there exists $t_0$ such that: for $t\geq t_0$ and $m \geq t$,
\begin{align*}
\sum_{\substack{h'_{<m}\in \mathcal{T}, a'_m\in\mathcal A}}
\nu (e'_{t:m-1}  \mid  h_{<t} \parallel a'_{t:m-1})
\cdot
\delta_Q (h'_{<m} a'_m)&
\\
<
2\beta \left( \tau/{|\mathcal A|} \right)^{-(m-t+1)}&
.
\end{align*}
\end{restatable}

We can then show analogues of \cref{thm:error-one-step,thm:one-step-and-global,thm:aiqi-eps-optimal} for $\eps$-greedy Self-AIXI, with our proof for AIQI repeated almost verbatim.
Note that the error term $2\tau$ would be introduced by the analogue of \cref{thm:error-one-step}, and thus $\pi_S$ can only be asymptotically $\eps$-optimal, where $\eps$ depends on the exploration rate $\tau$.

\section{Related Work}
\label{sec:related}

Compress and Control \citep[CNC;][]{veness2015compress} is an algorithm that shares many similarities to AIQI such as return prediction, but has only been analyzed for fully observable MDPs.
Both CNC and AIQI are on-policy Monte Carlo control methods with a distributional flavour \citep{bellemare2017distributional}, since they predict the return distribution instead of the average return.
CNC uses a model-based approach for MC evaluation, unlike AIQI.

Feature reinforcement learning \citep{daswani2015generic} converts a general environment to a finite state MDP through a learned feature map, and then performs planning \citep{hutter2009feature} or model-free RL \citep{nguyen2011feature,daswani2013q} inside the MDP.
\citet{hutter2014extreme} shows that any general environment can be mapped to an MDP where the optimal policy is still $\eps$-optimal in the general environment.
However, these theoretical results do not tell us about the learning of such mappings, and results that are concerned with the learning \citep{sunehag2010consistency} do not concern themselves with control.


Self-AIXI \citep{catt2023self} is a variant of AIXI that learns both a model of the environment and a model of itself. At every step, it uses both models to estimate its own action-values, by simulating the model of itself interacting with the model of the environment.
It then chooses the action with the highest estimated value.
Although Self-AIXI gets rid of planning, it is clearly model-based in that it learns and uses an environment model.
Our work was inspired by Self-AIXI, and can be seen as a modification that further gets rid of environment models by directly predicting action-values.

Finally, Optimal Direct Policy Search \citep[ODPS;][]{glasmachers2011optimal} is a policy search method that enumerates computable policies (or more generally policies in some arbitrary class) and evaluates them by repeated interaction with the environment. Like AIQI, ODPS is model-free and avoids planning. However, ODPS assumes an episodic POMDP setting with bounded episode lengths, whereas our results apply to general environments.

\section{Discussion}
\label{sec:discussion}

\subsection{Proof Technique}

To prove that AIQI is strong asymptotically $\eps$-optimal, we first showed that the action-value estimates converge to the true $\nu^\pi$ action-values (\cref{thm:error-q-estimation}), and that choosing the action with the highest (estimated) $\nu^\pi$ value makes the policy converge to optimal actions (\cref{thm:aiqi-eps-optimal}).
There are two key novelties in our proofs compared to existing work in universal AI.
First, we had to deal with the problem of delayed feedback using periodic augmentations. This was not needed in previous works which model policies or environments, since actions and percepts are never delayed, unlike $H$-step returns.
Second, we introduced a chain of inequalities using \cref{thm:one-step-and-global}, to make the $\delta_\infty$ term in the upper bound vanish with $\gamma^L$.
This allowed us to bypass ad-hoc assumptions such as off-policy sensibility
\citep[][Theorem~16]{catt2023self}.
In fact, we show in \cref{sec:self-aixi} that our technique can be applied to Self-AIXI to show analogous results without off-policy sensibility.
Overall, our work provides a blueprint for the analysis of policy iteration algorithms in general environments.




\subsection{Continual Reinforcement Learning}

The class of general environments is equivalent to the class of infinite-state POMDPs, which might seem unnecessarily large.
However, there are nontrivial settings that cannot be captured by finite-state POMDPs, one such example being continual RL \citep{ring1994continual}.
Continual RL studies environments that never stop changing \citep{khetarpal2022towards}, and where agents must keep learning in order to act optimally \citep{abel2023definition}.
The theory of universal AI sets aside problems with continual \emph{learning}, e.g., catastrophic forgetting, by assuming a perfect Bayesian learner with a broad prior.
This lets it focus on a different question:
how to design principled \emph{objectives} for continual RL, assuming continual learning is approximately solved.
Our work shows that learning distributional action-values is one such objective.



\subsection{Agent Foundations}

The theory of universal AI is central to agent foundations research \citep{soares2017agent}, since it provides formal models of powerful goal-directed agents under minimal assumptions about the environment.
As such, it has been used to analyze important issues in AI safety in a theoretically grounded way
\citep{everitt2016self,orseau2016safely,majha2019categorizing,everitt2018alignment,cohen2021curiosity}.
Our results show that asymptotic optimality in general reinforcement learning does not require an explicit world model.
We therefore view AIQI as a new theoretical object that may help clarify which ingredients of general-agent theory matter for safety.

Although our paper set out to primarily show that AIQI and AIXI are similar, they might differ in other ways, e.g., the off-policy behavior as we discussed in \cref{sec:off-policy}.
Another interesting difference is that, unlike AIXI, AIQI needs to deal with self-reference (\cref{def:grain-of-truth}), which is in turn related to embedded agency \citep{demski2019embedded,orseau2012space}.

\section{Conclusion}
\label{sec:conclusion}

We introduced Universal AI with Q-Induction (AIQI), the first model-free universal agent. AIQI performs universal induction over distributional action-value functions---rather than over environments or policies as in prior work.
It is essentially a Monte Carlo control algorithm, the most primitive form of model-free RL.
We prove that AIQI is strong asymptotically $\eps$-optimal, asymptotically $\eps$-Bayes-optimal, and, as expected from an on-policy MC control algorithm, not self-optimizing.
Remarkably, the only assumption we make is grain of truth, which is standard in universal AI and has a known resolution via reflective oracles.

Our findings significantly broaden the known landscape of universal agents.
We hope this provides both a theoretical foundation for understanding MC control with sequence models,
and a blueprint for the analysis of other policy iteration algorithms in general environments.
Promising directions for future work include
analyzing the properties of AIQI with exploration strategies that are more sophisticated than $\eps$-greedy, such as knowledge-seeking agents \citep{orseau2014universal},
investigating similarities/differences between AIQI and AIXI,
and exploring other possible forms of model-free universal AI such as those that involve policy search.


\begin{acknowledgements}
This work was supported by Institute for Information \& communications Technology Planning \& Evaluation(IITP) grant funded by the Korea government(MSIT) (RS-2019-II190075, Artificial Intelligence Graduate School Program(KAIST); RS-2024-00509279, Fundamental Research in Artificial Intelligence).

We thank Cole Wyeth for carefully reading the paper and discussing ways to improve it. We also thank the Universal Algorithmic Intelligence community for inviting us for a talk and providing insightful questions/feedback.
\end{acknowledgements}

\bibliography{bibliography}

@book{hutter2005universal,
  title={Universal artificial intelligence: Sequential decisions based on algorithmic probability},
  author={Hutter, Marcus},
  volume={300},
  year={2005},
  publisher={Springer}
}

@book{ring1994continual,
  title={Continual learning in reinforcement environments},
  author={Ring, Mark Bishop},
  year={1994},
  publisher={The University of Texas at Austin}
}

@article{abel2023definition,
  title={A definition of continual reinforcement learning},
  author={Abel, David and Barreto, Andr{\'e} and Van Roy, Benjamin and Precup, Doina and van Hasselt, Hado P and Singh, Satinder},
  journal={Advances in Neural Information Processing Systems},
  volume={36},
  pages={50377--50407},
  year={2023}
}

@article{khetarpal2022towards,
  title={Towards continual reinforcement learning: A review and perspectives},
  author={Khetarpal, Khimya and Riemer, Matthew and Rish, Irina and Precup, Doina},
  journal={Journal of Artificial Intelligence Research},
  volume={75},
  pages={1401--1476},
  year={2022}
}

@inproceedings{orseau2012memory,
  title={Memory issues of intelligent agents},
  author={Orseau, Laurent and Ring, Mark},
  booktitle={International Conference on Artificial General Intelligence},
  pages={219--231},
  year={2012},
  organization={Springer}
}

@inproceedings{everitt2016self,
  title={Self-modification of policy and utility function in rational agents},
  author={Everitt, Tom and Filan, Daniel and Daswani, Mayank and Hutter, Marcus},
  booktitle={International conference on artificial general intelligence},
  pages={1--11},
  year={2016},
  organization={Springer}
}

@article{solomonoff1964formal,
  title={A formal theory of inductive inference. Part I},
  author={Solomonoff, Ray J},
  journal={Information and control},
  volume={7},
  number={1},
  pages={1--22},
  year={1964},
  publisher={Elsevier}
}

@article{willems1995context,
  title={The context-tree weighting method: basic properties},
  author={Willems, FMJ and Shtarkov, Yu M and Tjalkens, TJ},
  journal={IEEE Transactions on Information Theory},
  volume={41},
  number={3},
  pages={653--664},
  year={1995},
  publisher={Institute of Electrical and Electronics Engineers}
}

@article{solomonoff1964formal2,
  title={A formal theory of inductive inference. Part II},
  author={Solomonoff, Ray J},
  journal={Information and control},
  volume={7},
  number={2},
  pages={224--254},
  year={1964},
  publisher={Academic Press}
}

@article{legg2007universal,
  title={Universal intelligence: A definition of machine intelligence},
  author={Legg, Shane and Hutter, Marcus},
  journal={Minds and machines},
  volume={17},
  number={4},
  pages={391--444},
  year={2007},
  publisher={Springer}
}

@article{williams1992simple,
  title={Simple statistical gradient-following algorithms for connectionist reinforcement learning},
  author={Williams, Ronald J},
  journal={Machine learning},
  volume={8},
  number={3},
  pages={229--256},
  year={1992},
  publisher={Springer}
}

@article{catt2023self,
  title={Self-predictive universal AI},
  author={Catt, Elliot and Grau-Moya, Jordi and Hutter, Marcus and Aitchison, Matthew and Genewein, Tim and Deletang, Gregoire and Li, Kevin and Veness, Joel},
  journal={Advances in Neural Information Processing Systems},
  volume={36},
  pages={27181--27198},
  year={2023}
}

@article{cohen2019strongly,
  title={A strongly asymptotically optimal agent in general environments},
  author={Cohen, Michael K and Catt, Elliot and Hutter, Marcus},
  journal={arXiv preprint arXiv:1903.01021},
  year={2019}
}

@article{leike2016thompson,
  title={Thompson sampling is asymptotically optimal in general environments},
  author={Leike, Jan and Lattimore, Tor and Orseau, Laurent and Hutter, Marcus},
  journal={arXiv preprint arXiv:1602.07905},
  year={2016}
}

@inproceedings{lattimore2014bayesian,
  title={Bayesian reinforcement learning with exploration},
  author={Lattimore, Tor and Hutter, Marcus},
  booktitle={International conference on algorithmic learning theory},
  pages={170--184},
  year={2014},
  organization={Springer}
}

@article{orseau2014universal,
  title={Universal knowledge-seeking agents},
  author={Orseau, Laurent},
  journal={Theoretical Computer Science},
  volume={519},
  pages={127--139},
  year={2014},
  publisher={Elsevier}
}

@article{hutter2009feature,
  title={Feature Reinforcement Learning: Part I. Unstructured MDPs},
  author={Hutter, Marcus},
  journal={arXiv preprint arXiv:0906.1713},
  year={2009}
}

@article{konda1999actor,
  title={Actor-critic algorithms},
  author={Konda, Vijay and Tsitsiklis, John},
  journal={Advances in neural information processing systems},
  volume={12},
  year={1999}
}

@inproceedings{veness2015compress,
  title={Compress and control},
  author={Veness, Joel and Bellemare, Marc and Hutter, Marcus and Chua, Alvin and Desjardins, Guillaume},
  booktitle={Proceedings of the AAAI Conference on Artificial Intelligence},
  volume={29},
  year={2015}
}

@book{rummery1994line,
  title={On-line Q-learning using connectionist systems},
  author={Rummery, Gavin A and Niranjan, Mahesan},
  volume={37},
  year={1994},
  publisher={University of Cambridge, Department of Engineering Cambridge, UK}
}

@article{veness2011monte,
  title={A monte-carlo aixi approximation},
  author={Veness, Joel and Ng, Kee Siong and Hutter, Marcus and Uther, William and Silver, David},
  journal={Journal of Artificial Intelligence Research},
  volume={40},
  pages={95--142},
  year={2011}
}

@book{leike2016nonparametric,
  title={Nonparametric general reinforcement learning},
  author={Leike, Jan},
  year={2016},
  publisher={The Australian National University (Australia)}
}

@book{hutter2024introduction,
  title={An introduction to universal artificial intelligence},
  author={Hutter, Marcus and Quarel, David and Catt, Elliot},
  year={2024},
  publisher={Chapman and Hall/CRC}
}

@incollection{everitt2018universal,
  title={Universal artificial intelligence: Practical agents and fundamental challenges},
  author={Everitt, Tom and Hutter, Marcus},
  booktitle={Foundations of trusted autonomy},
  pages={15--46},
  year={2018},
  publisher={Springer}
}

@phdthesis{watkins1989learning,
  title        = {Learning from delayed rewards},
  author       = {Watkins, Christopher John Cornish Hellaby},
  year         = {1989},
  school       = {King's College, University of Cambridge},
  address      = {Cambridge, United Kingdom}
}

@phdthesis{catt2022foundations,
  title        = {On the Foundations of Universal Artificial Intelligence},
  author       = {Catt, Elliot},
  year         = {2022},
  school       = {The Australian National University},
  address      = {Canberra, Australia}
}

@article{leike2016formal,
  title={A formal solution to the grain of truth problem},
  author={Leike, Jan and Taylor, Jessica and Fallenstein, Benya},
  journal={arXiv preprint arXiv:1609.05058},
  year={2016}
}

@article{meulemans2025embedded,
  title={Embedded Universal Predictive Intelligence: a coherent framework for multi-agent learning},
  author={Meulemans, Alexander and Nasser, Rajai and Wo{\l}czyk, Maciej and Weis, Marissa A and Kobayashi, Seijin and Richards, Blake and Lajoie, Guillaume and Steger, Angelika and Hutter, Marcus and Manyika, James and others},
  journal={arXiv preprint arXiv:2511.22226},
  year={2025}
}

@book{von1944theory,
  title={Theory of Games and Economic Behavior},
  author={Von Neumann, John and Morgenstern, Oskar},
  year={1944},
  publisher={Princeton University Press},
  address={Princeton}
}

@book{bellman1957dynamic,
  title={Dynamic Programming},
  author={Bellman, Richard Ernest},
  year={1957},
  publisher={Courier Dover Publications},
  address={Mineola, NY}
}

@inproceedings{lattimore2011asymptotically,
  title={Asymptotically optimal agents},
  author={Lattimore, Tor and Hutter, Marcus},
  booktitle={International Conference on Algorithmic Learning Theory},
  pages={368--382},
  year={2011},
  organization={Springer}
}

@article{blackwell1962merging,
  title={Merging of opinions with increasing information},
  author={Blackwell, David and Dubins, Lester},
  journal={The Annals of Mathematical Statistics},
  volume={33},
  number={3},
  pages={882--886},
  year={1962},
  publisher={JSTOR}
}

@article{kalai1993rational,
  title={Rational learning leads to Nash equilibrium},
  author={Kalai, Ehud and Lehrer, Ehud},
  journal={Econometrica: Journal of the Econometric Society},
  pages={1019--1045},
  year={1993},
  publisher={JSTOR}
}

@inproceedings{hutter2002self,
  title={Self-optimizing and Pareto-optimal policies in general environments based on Bayes-mixtures},
  author={Hutter, Marcus},
  booktitle={International Conference on Computational Learning Theory},
  pages={364--379},
  year={2002},
  organization={Springer}
}

@inproceedings{nguyen2011feature,
  title={Feature reinforcement learning in practice},
  author={Nguyen, Phuong and Sunehag, Peter and Hutter, Marcus},
  booktitle={European Workshop on Reinforcement Learning},
  pages={66--77},
  year={2011},
  organization={Springer}
}

@inproceedings{hutter2014extreme,
  title={Extreme state aggregation beyond MDPs},
  author={Hutter, Marcus},
  booktitle={International Conference on Algorithmic Learning Theory},
  pages={185--199},
  year={2014},
  organization={Springer}
}

@inproceedings{daswani2013q,
  title={Q-learning for history-based reinforcement learning},
  author={Daswani, Mayank and Sunehag, Peter and Hutter, Marcus},
  booktitle={Asian Conference on Machine Learning},
  pages={213--228},
  year={2013},
  organization={PMLR}
}

@phdthesis{daswani2015generic,
  title={Generic Reinforcement Learning Beyond Small MDPs},
  author={Daswani, Mayank},
  year={2015},
  school={Australian National University}
}

@inproceedings{fallenstein2015reflective,
  title={Reflective oracles: A foundation for game theory in artificial intelligence},
  author={Fallenstein, Benja and Taylor, Jessica and Christiano, Paul F},
  booktitle={International Workshop on Logic, Rationality and Interaction},
  pages={411--415},
  year={2015},
  organization={Springer}
}

@inproceedings{bellemare2017distributional,
  title={A distributional perspective on reinforcement learning},
  author={Bellemare, Marc G and Dabney, Will and Munos, R{\'e}mi},
  booktitle={International conference on machine learning},
  pages={449--458},
  year={2017},
  organization={Pmlr}
}

@inproceedings{sunehag2010consistency,
  title={Consistency of feature Markov processes},
  author={Sunehag, Peter and Hutter, Marcus},
  booktitle={International Conference on Algorithmic Learning Theory},
  pages={360--374},
  year={2010},
  organization={Springer}
}

@phdthesis{lattimore2014theory,
  title={Theory of General Reinforcement Learning},
  author={Lattimore, Tor},
  year={2014},
  school={Australian National University}
}

@inproceedings{coulom2006efficient,
  title={Efficient selectivity and backup operators in Monte-Carlo tree search},
  author={Coulom, R{\'e}mi},
  booktitle={International conference on computers and games},
  pages={72--83},
  year={2006},
  organization={Springer}
}

@article{wyeth2025limit,
  title={Limit-Computable Grains of Truth for Arbitrary Computable Extensive-Form (Un) Known Games},
  author={Wyeth, Cole and Hutter, Marcus and Leike, Jan and Taylor, Jessica},
  journal={arXiv preprint arXiv:2508.16245},
  year={2025}
}

@article{cohen2021curiosity,
  title={Curiosity killed or incapacitated the cat and the asymptotically optimal agent},
  author={Cohen, Michael K and Catt, Elliot and Hutter, Marcus},
  journal={IEEE Journal on Selected Areas in Information Theory},
  volume={2},
  number={2},
  pages={665--677},
  year={2021},
  publisher={IEEE}
}

@inproceedings{orseau2016safely,
  title={Safely interruptible agents},
  author={Orseau, Laurent and Armstrong, Stuart},
  booktitle={Proceedings of the Thirty-Second Conference on Uncertainty in Artificial Intelligence},
  pages={557--566},
  year={2016}
}

@article{majha2019categorizing,
  title={Categorizing Wireheading in Partially Embedded Agents},
  author={Majha, Arushi and Sarkar, Sayan and Zagami, Davide},
  journal={arXiv preprint arXiv:1906.09136},
  year={2019}
}

@article{everitt2018alignment,
  title={The alignment problem for Bayesian history-based reinforcement learners},
  author={Everitt, Tom and Hutter, Marcus},
  journal={Under submission},
  year={2018}
}

@article{demski2019embedded,
  title={Embedded agency},
  author={Demski, Abram and Garrabrant, Scott},
  journal={arXiv preprint arXiv:1902.09469},
  year={2019}
}

@incollection{soares2017agent,
  title={Agent foundations for aligning machine intelligence with human interests: a technical research agenda},
  author={Soares, Nate and Fallenstein, Benya},
  booktitle={The technological singularity: Managing the journey},
  pages={103--125},
  year={2017},
  publisher={Springer}
}

@inproceedings{orseau2012space,
  title={Space-time embedded intelligence},
  author={Orseau, Laurent and Ring, Mark},
  booktitle={International Conference on Artificial General Intelligence},
  pages={209--218},
  year={2012},
  organization={Springer}
}

@inproceedings{glasmachers2011optimal,
  title={Optimal direct policy search},
  author={Glasmachers, Tobias and Schmidhuber, J{\"u}rgen},
  booktitle={International Conference on Artificial General Intelligence},
  pages={52--61},
  year={2011},
  organization={Springer}
}

@misc{wyeth2025unbounded,
  author       = {Wyeth, Cole},
  title        = {Unbounded Embedded Agency: {AEDT} w.r.t. {rOSI}},
  year         = {2025},
  month        = jul,
  day          = {20},
  howpublished = {\url{https://www.lesswrong.com/posts/B6gumHyuxzR5yn5tH/unbounded-embedded-agency-aedt-w-r-t-rosi}},
  note         = {LessWrong blog post. Accessed: 2026-06-02}
}

\newpage

\onecolumn

\title{\mytitle\\(Supplementary Material)}
\maketitle

\appendix

\section{Notation}
\label{sec:notation}

\begin{table}[ht]
    \centering
    \caption{Summary of Notations}
    \label{tab:notations}
    \begin{tabular}{@{}ll@{}}
        \toprule
        \textbf{Symbol} & \textbf{Description} \\
        \midrule
        $\epsilon$ & The empty string, e.g., initial history $h_{<1}$ \\
        $\mathbb{N}$ & Set of natural numbers $\{0, 1, 2, \dots\}$ \\
        $\mathbb{Z}_+$ & Set of positive integers $\{1, 2, 3, \dots\}$ \\
        $\mathcal{A}, \mathcal{O}, \mathcal{R}$ & Sets of actions, observations, and rewards \\
        $\mathcal{E}$ & Percept space, defined as $\mathcal{O} \times \mathcal{R}$ \\
        $\mathcal{H}$ & The set of all finite histories $(\mathcal{A} \times \mathcal{E})^\ast $ \\
        $h_{<t}$ & History sequence $a_1 e_1 \dots a_{t-1} e_{t-1}$ \\
        $\pi, \nu$ & Policy $\pi: \mathcal{H} \to \Delta\mathcal{A}$ and environment $\nu: \mathcal{H} \times \mathcal{A} \to \Delta\mathcal{E}$ \\
        $V^\pi_\nu, Q^\pi_\nu$ & Value function and Q-value function \\
        $H(\eta)$ & $\eta$-effective horizon \\
        $M$             & Discretization level for returns \\
        $\mathcal{Z}$   & Alphabet of discretized returns $\{0, \frac{1}{M}, \dots, \frac{M-1}{M}\}$ \\
        $z_t$           & $M$-discretized $H$-step return at time $t$ \\
        $\tilde z_t$           & Augmented return at time $t$ in an augmented outcome \\
        $N$             & Augmentation period ($N \geq H$) \\
        $n$             & Phase of periodic augmentation ($n \in \{0, \dots, N-1\}$) \\
        $\tilde{\Omega}^{(n)}, \tilde{\mathcal{H}}^{(n)}$ & Phase $n$ augmented sample space and histories \\
        $\text{aug}_n$  & Mapping from outcomes to phase $n$-augmented outcomes \\
        $\psi_n$        & Mixture return-predictor for phase $n$ \\
        $\psi$          & Unified predictor that selects $\psi_n$ where $n = t \bmod N$ \\
        $\hat{Q}(h_{<t}, a_t)$ & AIQI estimated Q-value (expected discretized return) \\
        $\tau$          & Exploration probability \\
        $D(P, Q)$       & Total Variation (TV) distance between measures $P$ and $Q$ \\
        $L$             & Effective lookahead length ($N-H+1$) \\
        \bottomrule
    \end{tabular}
\end{table}

\section{Proofs}
\label{sec:proofs}

\convergencetv*
\begin{proof}
Recall that
\[
\lim_{t\to\infty} D(q_n, \tilde \nu^\pi _n  \mid  \tilde h\nth_{<t}) = 0,
\quad
\tilde \nu^\pi_n \text{-a.s.}
\]
In other words, there exists an event $\tilde S_n \subseteq \tilde \Omega\nth $, with $\tilde \nu^\pi_n (\tilde S_n)=1$, where $D(q_n, \tilde \nu^\pi _n  \mid  \tilde h\nth_{<t})$ converges to $0$ pointwise.
Let $F$ be the set of ``faulty'' outcomes in $\tilde S_n$ whose augmented returns $\tilde z_{i\equiv n}$ do not agree with the true returns $z_{i\equiv n}$.
Recall the definition of  $\tilde \nu^\pi_n$ as the pushforward of $\nu^\pi$ by $\aug_n$.
Since $F$ is disjoint from the image of $\aug_n$, we get $\tilde \nu^\pi_n (F)=0$.

Let $S_n$ be the preimage of ${\tilde S_n} - F$ under $\aug_n$. Then, by the injectivity of $\aug_n$, we see that $\nu^\pi(S_n) = \tilde\nu^\pi_n(\tilde S_n - F) = \tilde \nu^\pi_n (\tilde S_n)=1$.
Since pointwise convergence on $\tilde S_n$ implies pointwise convergence on the subset $\tilde S_n - F$, we see that $D(q_n, \tilde \nu^\pi _n  \mid  {{\aug_n}(h)}_{<t})$ converges pointwise to $0$ on $S_n$.
Construct such $S_n$ for every $n$, and let $S = \bigcap_n S_n$. Then, $\nu^\pi (S) =1$, and $D(q_n, \tilde \nu^\pi _n  \mid  {{\aug_n}(h)}_{<t})$ converges pointwise to $0$ on $S$ for all $n$.
\end{proof}

\errorreturnpredictor*
\begin{proof}
Fix some $\beta > 0$. By \cref{thm:convergence_tv}, we can choose some $S \subseteq \Omega $ with $\nu^\pi (S) =1$ such that $\forall h\in S$, $\exists t_0, \forall t\geq t_0, \forall n\in[0,N-1]$,
\begin{equation}
\label{eq:error-return-predictor-1}
D(q_n, \tilde \nu^\pi _n  \mid  {{\aug_n}(h)}_{<t}) < \beta
.
\end{equation}
Consider some $h \in S$ and $t\geq t_0$. For all $m\in [t, t{+}N{-}H]$ and $n=m\mod N$, we see that \cref{eq:condition_aug} is satisfied for both $t$ and $m$, so we can apply $\aug_n$ on histories of lengths $t-1$ and $m-1$.
Then,
\begin{align}
& \hspace{11pt} \sum_{\substack{h'_{<m} \in \mathcal T,\, a'_m \in \mathcal A, \\ \tilde z_{m}\in\mathcal Z}}
\nu^\pi \left( h'_{t:m-1} a'_m  \mid  h_{<t} \right)
\cdot
\big|
\psi (\tilde z_m  \mid  h'_{<m} a'_m)
-
\nu^\pi (\tilde z_m  \mid  h'_{<m} a'_m)
\big|
\nonumber\\
& \labelrel{=}{rel:errret:unified}
\sum_{\substack{h'_{<m} \in \mathcal T,\, a'_m \in \mathcal A, \\ \tilde z_{m}\in\mathcal Z}}
\nu^\pi \left( h'_{t:m-1} a'_m  \mid  h_{<t} \right)
\cdot
\big|
\psi_n (\tilde z_m  \mid  {\aug_n}(h'_{<m})\, a'_m)
-
\tilde \nu^\pi_n (\tilde z_m  \mid  {\aug_n}(h'_{<m})\, a'_m)
\big|
\nonumber\\
& \labelrel{=}{rel:errret:pushforward}
\sum_{\substack{h'_{<m} \in \mathcal T,\, a'_m \in \mathcal A, \\ \tilde z_{m}\in\mathcal Z}}
\tilde\nu_n^\pi \left( h'_{t:m-1} a'_m  \mid  {\aug_n}(h_{<t}) \right)
\cdot
\big|
\psi_n (\tilde z_m  \mid  {\aug_n}(h_{<t})\, h'_{t:m-1} a'_m)
-
\tilde \nu^\pi_n (\tilde z_m  \mid  {\aug_n}(h_{<t})\, h'_{t:m-1} a'_m)
\big|
\nonumber\\
& \labelrel{=}{rel:errret:joint}
\sum_{\substack{h'_{<m} \in \mathcal T,\, a'_m \in \mathcal A, \\ \tilde z_{m}\in\mathcal Z}}
\big |
q_n \left( h'_{t:m-1} a'_m \tilde z_m  \mid  {\aug_n}(h_{<t}) \right)
-
\tilde \nu^\pi_n \left( h'_{t:m-1} a'_m \tilde z_m  \mid  {\aug_n}(h_{<t}) \right)
\big |
\nonumber\\
& \labelrel{\leq}{rel:errret:partitiontv}
\hspace{2pt} 2 \cdot D(q_n, \tilde \nu^\pi _n  \mid  {{\aug_n}(h)}_{<t}).
\label{eq:error-return-predictor-2}
\end{align}

Equality \eqref{rel:errret:unified} is just rewriting the unified predictor $\psi$ as the phase $n$ predictor $\psi_n$ (and likewise rewriting the true conditional using $\tilde\nu_n^\pi$) via the definitions of $\psi$ and $\tilde\nu_n^\pi$.
Equality \eqref{rel:errret:pushforward} uses that $\aug_n(h_{<t})$ is a deterministic function of $h_{<t}$, and thus conditioning under $h_{<t}$ is the same as conditioning under $\aug_n(h_{<t})$.
Equality \eqref{rel:errret:joint} uses the definitions of $q_n$ and $\tilde \nu^\pi_n$, noting that there are no augmented returns between indices $t$ and $m$ other than $\tilde z_m$.
Inequality \eqref{rel:errret:partitiontv} is \cref{eq:partition-tv} applied to the disjoint cylinder sets formed by $h'_{<m} a'_m \tilde z_m$.

By the definition of AIQI, each action in $a'_{t:m}$ receives a probability of at least $\tau/|\mathcal A|$ from the AIQI policy $\pi$. Thus,
\begin{align}
\nu^\pi \left( h'_{t:m-1} a'_m  \mid  h_{<t} \right)
&= \prod_{i=t}^m \pi (a'_i  \mid  h'_{<i}) \prod_{i=t}^{m-1} \nu( e'_i  \mid  h'_{<i} a'_i)
\nonumber \\
& \geq
\left( \tau/|\mathcal A| \right)^{m-t+1}
\nu (e'_{t:m-1}  \mid  h_{<t} \parallel a'_{t:m-1}).
\label{eq:error-return-predictor-3}
\end{align}
By combining \cref{eq:error-return-predictor-1,eq:error-return-predictor-2,eq:error-return-predictor-3}, we obtain the desired result.
\end{proof}

\errorqestimation*
\begin{proof}
Let us define the lower-approximate Q-value as
\[
\underline{Q}_\nu^\pi (h_{<t},a_t)
:=
\mathbb E_{\nu^\pi} [z_t  \mid  h_{<t}a_t]
=
\sum_{z_t \in \mathcal Z}
z_t \cdot \nu^\pi (z_t  \mid  h_{<t},a_t)
.
\]
Recall that the $H$-step Q-value is
$
Q_{\nu,H}^{\pi} (h_{<t},a_t) = \mathbb E_{\nu^\pi} [\ret_{t,H}  \mid  h_{<t}a_t].
$
Since $|z_t - \ret_{t,H}|<1/M$, we see that $| \underline{Q}_\nu^\pi(\cdot) - Q_{\nu,H}^{\pi}(\cdot)|<1/M$.
Since we set $H$ to be the $\eta$-effective horizon $H(\eta)$, $|\ret_{t,H} - \ret_t| \leq \eta$, so that $|Q_{\nu,H}^{\pi}(\cdot) - Q_\nu^\pi(\cdot)|\leq\eta$.
Thus we obtain
$
| \underline{Q}_\nu^\pi (\cdot) - Q_\nu^\pi(\cdot) | < M^{-1} + \eta
$, and
\begin{equation}
\sum_{h'_{<m}\in \mathcal{T}}
\nu (e'_{t:m-1}  \mid  h_{<t} \parallel a'_{t:m-1})
\cdot
|
\underline{Q}^\pi_\nu (h'_{<m} a'_{m})
-
Q_\nu^\pi (h'_{<m} a'_{m})
|
<
M^{-1} + \eta
,
\label{eq:error-q-estimation-1}
\end{equation}
where we used \cref{eq:nu-sum-one}.

Recall the definition of $\hat Q$ in \cref{def:aiqi}.
\begin{align}
& \hspace{1.15em} \sum_{\substack{h'_{<m}\in \mathcal{T}, \\ a'_m\in\mathcal A}}
\nu (e'_{t:m-1}  \mid  h_{<t} \parallel a'_{t:m-1})
\cdot
|
\hat Q (h'_{<m}\, a'_{m})
-
\underline Q^\pi_\nu (h'_{<m}\, a'_{m})
|
\nonumber\\
& \labelrel{=}{rel:errq:expand}
\sum_{\substack{h'_{<m}\in \mathcal{T}, \\ a'_m\in\mathcal A}}
\nu (e'_{t:m-1}  \mid  h_{<t} \parallel a'_{t:m-1})
\cdot
|
\sum_{\tilde z_{m}\in\mathcal Z}
\tilde z_{m}
\left[
\psi (\tilde z_{m}  \mid  h'_{<m}\, a'_{m})
-
\nu^\pi (\tilde z_{m}  \mid  h'_{<m}\, a'_{m})
\right]
|
\nonumber\\
& \labelrel{\leq}{rel:errq:triangle}
\sum_{\substack{h'_{<m}\in \mathcal{T},\\ a'_m\in\mathcal A,\, \tilde z_m \in \mathcal Z}}
\nu (e'_{t:m-1}  \mid  h_{<t} \parallel a'_{t:m-1})
\cdot
\tilde z_{m} \cdot
|
\psi (\tilde z_{m}  \mid  h'_{<m}\, a'_{m})
-
\nu^\pi (\tilde z_{m}  \mid  h'_{<m}\, a'_{m})
|
\nonumber\\
& \labelrel{\leq}{rel:errq:dropz}
\sum_{\substack{h'_{<m}\in \mathcal{T},\\ a'_m\in\mathcal A,\, \tilde z_m \in \mathcal Z}}
\nu (e'_{t:m-1}  \mid  h_{<t} \parallel a'_{t:m-1})
\cdot
|
\psi (\tilde z_{m}  \mid  h'_{<m}\, a'_{m})
-
\nu^\pi (\tilde z_{m}  \mid  h'_{<m}\, a'_{m})
|
\nonumber\\
& =
\sum_{{\substack{h'_{<m}\in \mathcal{T}, \\ a'_m\in\mathcal A}}}
\nu (e'_{t:m-1}  \mid  h_{<t} \parallel a'_{t:m-1})
\cdot
\delta_\psi(h'_{<m} a'_m)
\nonumber\\
& \labelrel{<}{rel:errq:apply-errret}
2\beta \left( \tau/|\mathcal A| \right)^{-(m-t+1)}.
\label{eq:error-q-estimation-2}
\end{align}
Equality \eqref{rel:errq:expand} expands $\hat Q$ and $\underline Q_\nu^\pi$ as expectations over $\tilde z_m$ (by definition).
Inequality \eqref{rel:errq:triangle} is the triangle inequality.
Inequality \eqref{rel:errq:dropz} uses $\tilde z_m\in[0,1]$.
Finally, \eqref{rel:errq:apply-errret} is an application of \cref{thm:error-return-predictor}.

Combining \cref{eq:error-q-estimation-1,eq:error-q-estimation-2} with triangle inequality, we obtain the desired result.
\end{proof}

\erroronestep*
\begin{proof}
\begin{align*}
\deltaone (h'_{<m})
& = |\max_a Q(h'_{<m},a) - V(h'_{<m})| \\
& \labelrel{=}{rel:err1:expandV}
\Bigl| \max_a Q(h'_{<m},a) - \bigl[ (1-\tau) Q(h'_{<m}, \argmax_a \hat Q) + (\tau/|\mathcal{A}|) \sum_{a} Q(h'_{<m},a) \bigr] \Bigr| \\
& \labelrel{\leq}{rel:err1:exploration-gap}
| \max_a Q(h'_{<m},a) - Q(h'_{<m}, \argmax_a \hat Q) | + 2\tau \\
& \leq | \max_a Q(h'_{<m},a) - \max_a \hat Q(h'_{<m},a) | + | \max_a \hat Q(h'_{<m},a) - Q(h'_{<m}, \argmax_a \hat Q) | + 2\tau \\
& \labelrel{\leq}{rel:err1:maxbound}
\max_a | Q(h'_{<m},a) - \hat Q(h'_{<m},a) | + | \hat Q(h'_{<m}, a^\ddagger ) - Q(h'_{<m}, a^\ddagger ) | + 2\tau \\
& =
| Q(h'_{<m},a^\dagger) - \hat Q(h'_{<m},a^\dagger) | + | \hat Q(h'_{<m}, a^\ddagger ) - Q(h'_{<m}, a^\ddagger ) | + 2\tau,
\end{align*}
where $a^\dagger := \argmax_a | Q(h'_{<m},a) - \hat Q(h'_{<m},a) |$ and $a^\ddagger  := \argmax_a \hat Q(h'_{<m},a)$ are functions of $h'_{<m}$.

Equality \eqref{rel:err1:expandV} uses the Bellman equation $V(\cdot)=\sum_a \pi(a\mid \cdot)\, Q(\cdot,a)$ together with the AIQI policy definition (greedy with probability\ $1-\tau$ and uniform exploration with probability $\tau$).
Inequality \eqref{rel:err1:exploration-gap} upper-bounds the effect of the exploration mixture; since $Q\in[0,1]$, the mixture term can change the value by at most $2\tau$.
Inequality \eqref{rel:err1:maxbound} uses the standard bound $|\max f-\max g|\le \max|f-g|$ and that $a^\ddagger=\argmax_a \hat Q(h'_{<m},a)$.

Multiplying both sides by $\nu (e'_{t:m-1}  \mid  h_{<t} \parallel a'_{t:m-1})$ and summing, we obtain
\begin{align*}
& 
\hspace{1.33em} \sum_{h'_{<m}\in\mathcal T} \nu(e'_{t:m-1}  \mid  h_{<t} \parallel a'_{t:m-1}) \cdot \deltaone(h'_{<m})
\\
& 
\leq
\sum_{h'_{<m}\in\mathcal T} \nu(e'_{t:m-1}  \mid  h_{<t} \parallel a'_{t:m-1}) 
\left(
| Q(h'_{<m},a^\dagger) - \hat Q(h'_{<m},a^\dagger) | + | \hat Q(h'_{<m}, a^\ddagger ) - Q(h'_{<m}, a^\ddagger ) |
\right) 
\\
&
\hspace{12pt}
+ \sum_{h'_{<m}\in\mathcal T} \nu(e'_{t:m-1}  \mid  h_{<t} \parallel a'_{t:m-1}) \cdot 2\tau 
\\
& \labelrel{\leq}{rel:err1:sum-over-actions}
2\cdot \sum_{\substack{h'_{<m}\in \mathcal{T}, \\ a'_m\in\mathcal A}}
\nu (e'_{t:m-1}  \mid  h_{<t} \parallel a'_{t:m-1})
\cdot
|\hat Q (h'_{<m} a'_{m}) - Q_\nu^\pi (h'_{<m} a'_{m})|
+ 2 \tau
\\
&
\labelrel{<}{rel:err1:apply-errq}
2\cdot (2\beta \left( \tau/|\mathcal A|\right)^{-(m-t+1)} + M^{-1} + \eta) + 2\tau
.
\end{align*}
Inequality \eqref{rel:err1:sum-over-actions} uses that $a^\dagger(h'_{<m})$ and $a^\ddagger(h'_{<m})$ select particular actions, so summing over all $a'_m\in\mathcal A$ upper-bounds each choice.
Finally, \eqref{rel:err1:apply-errq} applies \cref{thm:error-q-estimation}.
\end{proof}

\onestepandglobal*
\begin{proof}
We omit the subscript $\nu$ under $Q$ and $V$ for brevity.
\begin{align*}
\deltainfty(h_{<t}) & = |\max_a Q^\ast (h_{<t}, a) - V^\pi(h_{<t})|
\nonumber \\
&\labelrel{\leq}{rel:osg:tri}
|\max_a Q^\ast (h_{<t}, a) - \max_a Q^\pi(h_{<t}, a)| + |\max_a Q^\pi(h_{<t}, a) - V^\pi(h_{<t})|
\nonumber \\
&\labelrel{\leq}{rel:osg:maxdiff}
\max_a |Q^\ast (h_{<t}, a) - Q^\pi(h_{<t}, a)| + \deltaone(h_{<t})
\nonumber \\
&\labelrel{=}{rel:osg:choose-a}
\left|Q^\ast (h_{<t}, a_t') - Q^\pi(h_{<t}, a_t') \right| + \deltaone(h_{<t})\\
&\labelrel{=}{rel:osg:bellman}
\left| \sum_{e'_t} \nu(e_t | h_{<t} a_t') \left[ \left((1-\gamma) r_t + \gamma  V^\ast (h_{<t} a'_t e'_t)\right) - \left((1-\gamma) r_t + \gamma  V^\pi(h_{<t} a'_t e'_t)\right) \right] \right| + \deltaone(h_{<t})
\nonumber\\
&= \left| \gamma \sum_{e'_t} \nu(e_t | h_{<t} a_t') \left[ V^\ast (h_{<t} a'_t e'_t) -  V^\pi(h_{<t} a'_t e'_t) \right] \right| + \deltaone(h_{<t})
\nonumber\\
&\labelrel{\leq}{rel:osg:delta-infty}
\gamma \sum_{e'_t} \nu(e_t  \mid  h_{<t} a_t') \deltainfty(h_{<t} a'_t e'_t) + \deltaone(h_{<t})
\end{align*}
Inequality \eqref{rel:osg:tri} is the triangle inequality.
Inequality \eqref{rel:osg:maxdiff} uses $|\max f-\max g|\le \max |f-g|$ and the definition of $\deltaone$.
Equality \eqref{rel:osg:choose-a} is by the definition of $a'_t=\argmax_a |Q^\ast(h_{<t},a)-Q^\pi(h_{<t},a)|$ as stated in the lemma.
Equality \eqref{rel:osg:bellman} is the Bellman expansion of $Q^\ast$ and $Q^\pi$.
Finally, \eqref{rel:osg:delta-infty} uses $| \sum_x p_x u_x|\le \sum_x p_x |u_x|$ and $\deltainfty = V^\ast - V^\pi$.
\end{proof}

\strongimpliesbayes*
\begin{proof}
Proving that the inequality in our theorem holds $\xi^\pi$-almost surely implies that the inequality holds $\nu^\pi$-almost surely for any $\nu\in\mathcal M$. This is because, if $E$ is the set of all outcomes $h$ where the inequality \emph{doesn't} hold, $\nu^\pi(E) \leq \xi^\pi(E) / w(\nu) =0 $. Thus it suffices to prove that the inequality holds $\xi^\pi$-almost surely.

By the definition of the mixture $\xi$ and the value function $V$, we see that
\[
V_\xi^{\pi}(h_{<t})
=
\sum_{\nu\in\mathcal M} w(\nu \mid h_{<t})\, V_\nu^{\pi}(h_{<t})
.
\]
Let $\pi_\xi^\ast$ be an optimal policy in $\xi$. Then
\[
V_\xi^\ast(h_{<t})
= V_\xi^{\pi_\xi^\ast}(h_{<t})
= \sum_{\nu} w(\nu \mid h_{<t})\,V_\nu^{\pi_\xi^\ast}(h_{<t})
\leq \sum_{\nu} w(\nu \mid h_{<t})\,V_\nu^\ast(h_{<t}),
\]
so that
\begin{equation}
0 \leq V_\xi^\ast(h_{<t}) - V_\xi^\pi(h_{<t})
\leq
\sum_{\nu} w(\nu \mid h_{<t})\left(V_\nu^\ast(h_{<t}) - V_\nu^\pi(h_{<t})\right).
\label{eq:gap-upper}
\end{equation}
Thus it suffices to show that the right-hand side has $\limsup \leq \eps$, $\xi^\pi$-almost surely.

For each environment $\nu\in\mathcal M$, by \cref{thm:aiqi-eps-optimal} there exists an event $S_\nu\subseteq\Omega$
with $\nu^\pi(S_\nu)=1$ such that for all outcomes $h\in S_\nu$,
\begin{equation}
\textstyle
\limsup_{t\to\infty}\bigl(V_\nu^\ast(h_{<t})-V_\nu^\pi(h_{<t})\bigr)\leq\eps.
\label{eq:nu-gap}
\end{equation}

Define the posterior process $w_t(\nu):=w(\nu\mid h_{<t})$ and its limit
$w_\infty(\nu):=\lim_{t\to\infty} w_t(\nu)$, which exists $\xi^\pi$-a.s. since
$(w_t(\nu))_{t\geq 1}$ is a bounded martingale.
Let us consider the joint probability measure on $\mathcal M \times \Omega$, where we
first draw $\bm \nu\sim w$, then draw the outcome $h\sim \bm\nu^\pi$. Then $\xi^\pi$ is its marginal over $\Omega$.
A standard identity is the following:
\begin{equation}
\mathbb E [w_\infty (\nu) \mathds 1_E]
= \mathbb P(\bm{\nu}=\nu,h\in E)
= w(\nu) \nu^\pi(E).
\label{eq:joint-identity}
\end{equation}
Due to \cref{eq:joint-identity} and $\nu^\pi(S_\nu^\mathsf{c})=0$,
$\mathbb E_{\xi^\pi}[w_\infty(\nu)\mathds 1_{S_\nu^\mathsf{c}}]=0$. Since the integrand is nonnegative,
\begin{equation}
w_\infty(\nu) \mathds 1_{S_\nu^\mathsf{c}}=0,
\quad
\text{$\xi^\pi$-a.s.}
\label{eq:posterior-support}
\end{equation}
Hence, there exists a $\xi^\pi$-probability-one set, such that for any outcome $h$ in this set,
\cref{eq:nu-gap} holds for all $\nu$ where $w_\infty(\nu)>0$.

Now define the per-environment gap
\[
\delta_{t,\nu} := V_\nu^\ast(h_{<t}) - V_\nu^\pi(h_{<t}) \in [0,1],
\]
and the posterior-weighted gap
\[
G_t := \sum_{\nu} w_t(\nu) \delta_{t,\nu}.
\]
By \cref{eq:gap-upper}, it is enough to show $\limsup_{t\to\infty} G_t \leq \eps$, $\xi^\pi$-a.s.

Fix an outcome $h$ in a $\xi^\pi$-probability-one set on which $w_t(\nu)\to w_\infty(\nu)$ for all $\nu$
and \cref{eq:posterior-support} holds for all $\nu$.
Let $\alpha>0$. Since $w_\infty(\cdot)$ is a probability mass function on a countable set,
there exists a finite $F\subseteq\mathcal M$ such that $w_\infty(\nu)>0$ for all $\nu\in F$, and
\begin{equation}
\sum_{\nu\in F} w_\infty(\nu) \geq 1-\alpha.
\label{eq:finite-mass}
\end{equation}
Since $w_\infty(\nu)>0$, we have
$\limsup_{t\to\infty} \delta_{t,\nu} \leq \eps$ by \cref{eq:posterior-support}.
Therefore, for every such $\nu$ there exists $t_\nu$ such that for all $t\geq t_\nu$,
$
\delta_{t,\nu} \leq \eps+\alpha
$.
Let $t_0:=\max_{\nu\in F} t_\nu$. Then for all $t\geq t_0$,
\begin{equation}
\sum_{\nu\in F} w_t(\nu)\delta_{t,\nu}
\leq (\eps+\alpha)\sum_{\nu\in F} w_t(\nu)
\leq \eps+\alpha.
\label{eq:finite-part}
\end{equation}
Moreover, since $\delta_{t,\nu}\leq 1$,
$
\sum_{\nu\notin F} w_t(\nu)\delta_{t,\nu}
\leq \sum_{\nu\notin F} w_t(\nu)
$.
Because $F$ is finite and $w_t(\nu)\to w_\infty(\nu)$ pointwise, we have
$\sum_{\nu\in F} w_t(\nu)\to \sum_{\nu\in F} w_\infty(\nu)$, so by \cref{eq:finite-mass},
$
\sum_{\nu\notin F} w_t(\nu) = 1-\sum_{\nu\in F} w_t(\nu)
\to
1-\sum_{\nu\in F} w_\infty(\nu) \leq \alpha.
$
Hence there exists $t_1$ such that for all $t\geq t_1$,
\begin{equation}
\sum_{\nu\notin F} w_t(\nu) \leq 2\alpha.
\label{eq:tail-mass}
\end{equation}
Combining \cref{eq:finite-part,eq:tail-mass}, we obtain that for all $t\geq \max\{t_0, t_1\}$,
\[
G_t
= \sum_{\nu\in F} w_t(\nu)\delta_{t,\nu} + \sum_{\nu\notin F} w_t(\nu)\delta_{t,\nu}
\leq (\eps+\alpha) + 2\alpha.
\]
Since $\alpha>0$ is arbitrary, we conclude that
\[
\limsup_{t\to\infty} G_t \leq \eps,
\quad
\xi^\pi\text{-a.s.}
\qedhere
\]
\end{proof}

\aiqinotselfoptimizing*
\begin{proof}
Given $H\geq 2$, define
\[
c_H \;:=\; \max_{\gamma\in(0,1)} (1-\gamma)\gamma^{H-1}
\;=\; \frac{1}{H}\Big(\frac{H-1}{H}\Big)^{H-1}.
\]
Fix a tolerance $\eps\in(0,c_H)$.

Let
\[
\gamma \;:=\; \frac{H-1}{H}\in(0,1),
\qquad
c \;:=\; (1-\gamma)\gamma^{H-1} \;=\; c_H.
\]
Choose
\[
\delta \;:=\; \frac{1-\eps/c}{3}\in(0,1),
\qquad
K\in\mathbb N\ \text{s.t.}\ \frac{\tau}{K}\le \frac{\delta}{2}.
\]
Finally, assume $M$ is large enough so that discretization cannot collapse the strict value gap:
\begin{equation}\label{eq:M-choice-corrected}
\frac{1}{M} \;<\; c\frac{\delta}{2}
\qquad\text{equivalently}\qquad
M \;>\; \frac{2}{c\delta}.
\end{equation}

\paragraph{Environment.}
Let $\mathcal A=\{1,2,\dots,K\}$, $\mathcal O=\{\cdot,\star\}$, and $\mathcal R=\{0,\delta,1\}$.
The environment $\mu$ consists of episodes of length $H$.
For the first $H-1$ steps of each episode, output $(o,r)=(\cdot,0)$.
On the last step output $(o,r)=(\star,r)$ where $r$ depends on the $H$ actions taken in that episode:
\begin{itemize}
\item if all $H$ actions are $1$, then $r=1$;
\item if the first action is $1$ but not all $H$ actions are $1$, then $r=0$;
\item if the first action is not $1$, then $r=\delta$.
\end{itemize}
Then a new episode starts and the construction repeats forever.

\paragraph{Historic policy.}
Define $\pi'$ episode-wise as follows.
At the first step of an episode, sample $a$ uniformly from $\mathcal A$.
If this first action is not $1$, sample the remaining $H-1$ actions uniformly from $\mathcal A$.
If the first action is $1$, then for each of the remaining $H-1$ steps play $1$ with probability
\[
p \;:=\; \Big(\frac{\delta}{2}\Big)^{1/(H-1)},
\]
and otherwise sample uniformly from $\{2,\dots,K\}$.
Then conditioned on starting an episode with action $1$,
the probability of playing $1$ for all remaining $H-1$ steps is $p^{H-1}=\delta/2$.

\paragraph{Existence of a self-optimizing policy.}
Since $\mathcal M$ is a singleton, the $\mu$-optimal policy $\pi^*_\mu$ that always plays action $1$
is trivially self-optimizing for $(\mathcal M,\pi')$.

\paragraph{AIQI learns the $\pi'$-value.}
Under $\mu^{\pi'}$, the (periodically augmented) discretized return process is a computable
variable-order Markov model of order $O(H)$ (episode boundaries are marked by $\star$ and terminal rewards
depend only on the $H$ actions in the current episode).
Because $\psi$ is a Bayesian mixture over a class containing the true return-predictor for $\mu^{\pi'}$,
\citet{blackwell1962merging} implies that $\psi$'s posterior predictive distribution converges to the true
conditional return distribution under $\mu^{\pi'}$.
In particular, AIQI's value estimate
\[
\hat Q(h_{<t},a)=\sum_{\tilde z\in\mathcal Z}\tilde z\;\psi(\tilde z\mid h_{<t}a)
\]
converges to the true conditional expectation of the discretized target $z_t$ under $\mu^{\pi'}$:
\[
\underline{Q}^{\pi'}_{\mu}(h_{<t},a):=\mathbb E^{\pi'}_\mu[z_t\mid h_{<t},a_t=a].
\]

Now fix an episode-start time $t$ (immediately after observing $\star$).
In this environment, the only nonzero reward within the next $H$ steps is the terminal reward,
so the $H$-step discounted return from $t$ satisfies
\[
R_{t,H}=(1-\gamma)\sum_{k=0}^{H-1}\gamma^k r_{t+k}=(1-\gamma)\gamma^{H-1}r_{t+H-1}=c\,r_{t+H-1}.
\]
If we force $a_t\neq 1$ and then follow $\pi'$, the terminal reward is deterministically $\delta$, hence
\[
Q^{\pi'}_{\mu,H}(h_{<t},a\neq 1)=c\delta.
\]
If we force $a_t=1$ and then follow $\pi'$, the terminal reward is $1$ with probability $\delta/2$
and $0$ otherwise, hence
\[
Q^{\pi'}_{\mu,H}(h_{<t},a=1)=c\cdot\frac{\delta}{2}.
\]

Because AIQI optimizes the discretized target $z_t=\lfloor MR_{t,H}\rfloor/M$ and
$0\le R_{t,H}-z_t<1/M$ pointwise, we have for each action $a$:
\[
Q^{\pi'}_{\mu,H}(h_{<t},a)-\frac{1}{M}<\underline{Q}^{\pi'}_{\mu}(h_{<t},a)\le Q^{\pi'}_{\mu,H}(h_{<t},a).
\]
Therefore at episode starts,
\[
\underline{Q}^{\pi'}_{\mu}(h_{<t},a\neq 1)-\underline{Q}^{\pi'}_{\mu}(h_{<t},a=1)
\;\ge\; c\frac{\delta}{2}-\frac{1}{M}\;>\;0
\]
by~\eqref{eq:M-choice-corrected}.
So action $1$ is strictly suboptimal under the discretized $\pi'$-value at episode starts.

By $\mu^{\pi'}$-a.s.\ convergence of $\hat Q$ to $\underline{Q}^{\pi'}_\mu$, there exists a $t_0$ such that
for all episode-start times $t\ge t_0$, AIQI's exploitation choice $a^*(h_{<t})\neq 1$.

\paragraph{Lower-bound on the self-optimization gap.}
Fix any episode-start $t\ge t_0$.
Since $a^*(h_{<t})\neq 1$, AIQI chooses action $1$ at time $t$ with probability at most $\tau/K$
(it can only happen by uniform exploration).
Hence the expected terminal reward of the \emph{next} episode satisfies
\[
\mathbb E^{\hat\pi}_\mu[r_{t+H-1}\mid h_{<t}]
\;\le\; \delta\Big(1-\frac{\tau}{K}\Big)+1\cdot\frac{\tau}{K}
\;\le\; \delta+\frac{\tau}{K}
\;\le\; \frac{3\delta}{2}.
\]

Let $\pi^*$ denote the policy that always outputs action $1$.
Then $\pi^*$ achieves terminal reward $1$ at the end of \emph{every} episode, in particular
$\mathbb E^{\pi^*}_\mu[r_{t+H-1}\mid h_{<t}]=1$.
Because rewards are in $[0,1]$, the termwise differences at episode-terminal times are nonnegative, so
\begin{align*}
V^*_\mu(h_{<t})-V^{\hat\pi}_\mu(h_{<t})
&\ge V^{\pi^*}_\mu(h_{<t})-V^{\hat\pi}_\mu(h_{<t})\\
&=(1-\gamma)\sum_{j=0}^\infty \gamma^{jH+H-1}
\Big(1-\mathbb E^{\hat\pi}_\mu[r_{t+jH+H-1}\mid h_{<t}]\Big)\\
&\ge (1-\gamma)\gamma^{H-1}
\Big(1-\mathbb E^{\hat\pi}_\mu[r_{t+H-1}\mid h_{<t}]\Big)\\
&\ge c\Big(1-\frac{3\delta}{2}\Big).
\end{align*}
With our choice $\delta=(1-\eps/c)/3$, we have
\[
1-\frac{3\delta}{2}=1-\frac{1-\eps/c}{2}=\frac{1+\eps/c}{2},
\]
so
\[
V^*_\mu(h_{<t})-V^{\hat\pi}_\mu(h_{<t})
\;\ge\; c\cdot\frac{1+\eps/c}{2}
\;=\;\frac{c+\eps}{2}\;>\;\eps.
\]

Episode starts occur infinitely often, hence
\[
\limsup_{t\to\infty}\big(V^*_\mu(h_{<t})-V^{\hat\pi}_\mu(h_{<t})\big)>\eps
\qquad \mu^{\pi'}\text{-a.s.}
\]
Thus AIQI is not $\eps$-self-optimizing for $(\mathcal M,\pi')$.
\end{proof}

\section{Proofs for Self-Aixi}
\label{sec:proofs-self-aixi}

\selfaixitvconvergence*
\begin{proof}
Since $\nu^{\pi}$ is absolutely continuous w.r.t. $\xi^{\pi}$, we obtain by \citet{blackwell1962merging} that
\[
\lim_{t\to\infty} D\!\left(\xi^{\pi}, \nu^{\pi} \mid h_{<t}\right) = 0,
\quad \nu^{\pi}\text{-a.s.}
\]
Similarly, we obtain that
\[
\lim_{t\to\infty} D\!\left(\xi^{\zeta}, \xi^{\pi} \mid h_{<t}\right) = 0,
\quad \xi^{\pi}\text{-a.s.}
\]
$\xi^{\pi}$-almost sure convergence implies $\nu^{\pi}$-almost sure convergence, since $\nu^{\pi}$ is absolutely continuous w.r.t. $\xi^{\pi}$.
By triangle inequality for total variation distance, we obtain the desired result.
\end{proof}

\avgcondtv*
\begin{proof}
Define a probability measure \(R\) on \(\mathcal X \times \mathcal Y\) by
\[
R(x,dy) := P_X(x)\,Q(dy\mid x).
\]
Thus, \(R\) has the same marginal on $X$ as \(P\), but uses the conditional law of \(Y\mid X\) from \(Q\).
We can show that
\[
D(P,R)
=
\sum_{x\in\mathcal X}
P_X(x) \,D \bigl(P, Q \mid X=x \bigr),
\]
and that
\[
D(R,Q)=D(P_X,Q_X).
\]

Now apply the triangle inequality:
\[
D(P,R)\le D(P,Q)+D(R,Q).
\]
Substituting the identities above gives
\[
\sum_{x\in\mathcal X}
P_X(x) \,D \bigl(P, Q \mid X=x \bigr)
\;\le\;
D(P,Q)+D(P_X,Q_X).
\]

Finally, total variation distance is non-increasing under measurable maps. Applying this to the projection
$
(x,y)\mapsto x,
$
we get
$
D(P_X,Q_X)\le D(P,Q).
$
Therefore,
\[
D(P,Q)+D(P_X,Q_X)\le 2D(P,Q),
\]
which completes the proof.
\end{proof}

\exploringselfaixierrorqestimation*
\begin{proof}
Fix $\beta>0$.
By \cref{thm:self-aixi-tv-convergence}, there exists a $\nu^\pi$-probability-one set $S\subseteq\Omega$ such that for every $h\in S$ there is some $t_0$ with
\[
D(\nu^\pi , \xi^\zeta \mid h_{<t}) < \beta
\qquad\text{for all }t\ge t_0.
\]

Now fix $h\in S$, $t\ge t_0$, and $m\ge t$.
By \cref{thm:avg_cond_tv},
\[
\sum_{\substack{h'_{<m} \in \mathcal T,\\ a'_m \in \mathcal A}}
\nu^\pi \left( h'_{t:m-1} a'_m  \mid  h_{<t} \right)
\cdot
D(\nu^\pi, \xi^\zeta \mid h'_{<m} a'_m)
\leq
2 D(\nu^\pi, \xi^\zeta \mid h_{<t}),
\]
and by \cref{thm:self-aixi-q-value-difference-bound},
\[
\delta_Q(h'_{<m}a'_m)
=
\bigl|Q^\zeta_\xi(h'_{<m},a'_m)-Q^\pi_\nu(h'_{<m},a'_m)\bigr|
\le
D(\nu^\pi, \xi^\zeta \mid h'_{<m}a'_m).
\]
Hence
\[
\sum_{\substack{h'_{<m} \in \mathcal T,\\ a'_m \in \mathcal A}}
\nu^\pi \left( h'_{t:m-1} a'_m  \mid  h_{<t} \right)
\cdot
\delta_Q(h'_{<m}a'_m)
\leq
2D(\nu^\pi, \xi^\zeta \mid h_{<t}) < 2\beta.
\]

Since $\pi_S$ assigns probability of at least $\tau/|\mathcal A|$ to each action,
\begin{align*}
\nu^\pi \left( h'_{t:m-1} a'_m  \mid  h_{<t} \right)
&= \prod_{i=t}^m \pi (a'_i  \mid  h'_{<i}) \prod_{i=t}^{m-1} \nu( e'_i  \mid  h'_{<i} a'_i)
\\
& \geq
\left( \tau/|\mathcal A| \right)^{m-t+1}
\nu (e'_{t:m-1}  \mid  h_{<t} \parallel a'_{t:m-1}),
\end{align*}
and we obtain the desired result:
\[
\sum_{\substack{h'_{<m} \in \mathcal T,\\ a'_m \in \mathcal A}}
\nu (e'_{t:m-1}  \mid  h_{<t} \parallel a'_{t:m-1})
\cdot
\delta_Q (h'_{<m} a'_m)
<
2\beta \left( \tau/{|\mathcal A|} \right)^{-(m-t+1)}
.
\qedhere
\]
\end{proof}

\section{Multi-Agent Environment}
\label{sec:multi-agent-environment}

The single-agent results in \cref{sec:one-step-to-global} extend directly to multi-agent sequential interaction by viewing each agent as acting in a \emph{subjective environment} induced by the other agents \citep[Section 4.1]{leike2016formal}.
Consider $n$ agents interacting in a multi-agent environment
\[
\sigma : (\mathcal A^n \times \mathcal E^n)^* \times \mathcal A^n \to \Delta(\mathcal E^n),
\]
where at time $t$ the joint action and percept are
$\mathbf a_t := (a_t^1,\dots,a_t^n)$ and
$\mathbf e_t := (e_t^1,\dots,e_t^n)$.
Agent $i$ only observes its own action-percept history
\[
h^i_{<t} := a_1^i e_1^i \dots a_{t-1}^i e_{t-1}^i.
\]
Given $n$ policies $\pi_1,\dots,\pi_n$, let $\sigma_i$ denote the subjective environment of agent $i$, obtained by combining $\sigma$ with the policies of the other agents and marginalizing out the components not observed by agent $i$.
Together with $\pi_i$, this defines an ordinary environment-policy pair, so the formalism of the previous sections applies unchanged.

We say that $\pi_i$ is an $\eps$-best response at history $h^i_{<t}$ if
\[
V^\ast_{\sigma_i}(h^i_{<t}) - V^{\pi_i}_{\sigma_i}(h^i_{<t}) \le \eps.
\]
If this holds for every $i\in\{1,\dots,n\}$, then the policies $\pi_{1:n}$ is an $\eps$-Nash equilibrium at time $t$.

As in \cref{def:grain-of-truth}, the key self-referential requirement is that each agent's unified predictor $\psi_i$ has a grain of truth with respect to its own subjective environment $\sigma_i$.
A nontrivial instantiation can again be obtained by letting $\mathcal P^i_n$ be the class of all $O$-computable return-predictors for every $i$ and $n$, and assuming that the environment $\sigma$ is $O$-computable.

\begin{restatable}[AIQI converges to $\eps$-Nash equilibria]{theorem}{aiqimultiagent}
\label{thm:aiqi-multi-agent}
Fix a multi-agent environment $\sigma$ and a tolerance $\eps>0$.
For each agent $i\in\{1,\dots,n\}$, let
\[
\pi_i := \hat \pi^{H_i,M_i,N_i,\tau_i}_{\psi_i},
\]
where the parameters $H_i,M_i,N_i,\tau_i$ satisfy the conditions of \cref{thm:aiqi-eps-optimal} with tolerance $\eps/2$, and suppose that $\psi_i$ has a grain of truth with respect to the subjective environment $\sigma_i$ induced by $\sigma$ and the policies $\pi_{\neq i}$.
Then, the policies $\pi_{1:n}$ converge to an $\eps$-Nash equilibrium, $\sigma^{\pi_{1:n}}$-almost surely.
\end{restatable}

\begin{proof}
Fix an agent $i$.
Since $\sigma_i$ is an ordinary environment for agent $i$, \cref{thm:aiqi-eps-optimal} applied with tolerance $\eps/2$ gives
\[
\limsup_{t\to\infty}
\left(
V^\ast_{\sigma_i}(h^i_{<t}) - V^{\pi_i}_{\sigma_i}(h^i_{<t})
\right)
\le \eps/2,
\qquad
\sigma_i^{\pi_i}\text{-a.s.}
\]
Because $\sigma_i^{\pi_i}$ is exactly the marginal of the joint history distribution $\sigma^{\pi_{1:n}}$ on agent $i$'s coordinates, the same statement holds $\sigma^{\pi_{1:n}}$-almost surely.
Hence, for $\sigma^{\pi_{1:n}}$-almost every outcome $h$, there exists $t_i$ such that
\[
V^\ast_{\sigma_i}(h^i_{<t}) - V^{\pi_i}_{\sigma_i}(h^i_{<t}) \le \eps
\qquad
\text{for all } t\ge t_i.
\]
Since there are finitely many agents, intersecting these probability-one events over $i$ still yields a probability-one event.
Taking $t_0 := \max_i t_i$ proves that all agents are simultaneously $\eps$-best responses for all $t\ge t_0$, which is exactly the claim.
\end{proof}

\section{General Discount Sequences}
\label{sec:arbitrary_discount}

We can extend all our results to general time-consistent discount sequences \citep[p.~245]{hutter2024introduction} that decay faster than a geometric sequence.

First, we should introduce the relevant concepts. A time-consistent discount sequence is a sequence $\{ \gamma_k \}_{k=1}^\infty$ with $\gamma_k\geq0$ and $\sum_{k=1}^\infty \gamma_k < \infty$.
An example is the geometric discount sequence $\gamma_k = \gamma^k$.
The discount normalization factor is defined as $\Gamma_t = \sum_{k=t}^\infty \gamma_k$.
The full return at time $t$ is defined as
\[
R_t = \frac{1}{\Gamma_t} \sum_{k=t}^\infty \gamma_k r_k,
\]
and the value functions $V$ and $Q$ are defined with $R_t$ as in \cref{sec:grl}.
The value functions satisfy the general discount Bellman equations:
\[
V^\pi_\nu (h_{<t}) = \sum_{a'_t \in \mathcal{A}} \pi(a'_t \mid h_{<t})\, Q^\pi_\nu(h_{<t}, a'_t),
\]
\[
Q^\pi_\nu(h_{<t}, a_t) = \frac{1}{\Gamma_t} \sum_{e'_t \in \mathcal{E}} \nu(e'_t \mid h_{<t}a_t) \big[ \gamma_t {r_t}
+ \Gamma_{t+1} V^\pi_\nu(h_{<t} a_t e'_t) \big].
\]
$H$-step returns and values are defined as in \cref{sec:grl}.
We define the $\eta$-effective horizon at time $t$ as
\[
H_t(\eta) := \min \left\{ H \in \mathbb Z_+ \;\middle|\; \frac{\Gamma_{t+H}}{\Gamma_t} \leq \eta \right\}.
\]
One can easily show that the $H_t(\eta)$-step return at $t$ differs from the full return at $t$ by at most $\eta$.
If a discount sequence decays faster than a geometric sequence, i.e.,
\[
\limsup_{t\to\infty} \frac{\gamma_{t+1}}{\gamma_t} < \lambda
\]
for some $\lambda<1$, the effective horizon $H_t(\eta)$ for a fixed $\eta$ has an upper bound across all $t$.

We only need to modify our proofs and results in several places to accommodate the generalization.
\begin{itemize}
\item
The statement of \cref{thm:one-step-and-global} should change to
\[
\deltainfty(h_{<t}) \leq \frac{\Gamma_{t+1}}{\Gamma_t} \sum_{e'_t} \nu(e'_t  \mid  h_{<t} \, a'_t)\, \deltainfty(h_{<t} \,a'_t\, e'_t) + \deltaone(h_{<t}),
\]
which is just the original statement with $\gamma$ replaced by ${\Gamma_{t+1}}/{\Gamma_t}$.
The logic of its proof remains the same, except that we should use the general discount Bellman equation instead of the geometric discount Bellman equation, in one of its steps.
\item
The chain of inequalities resulting from \cref{thm:one-step-and-global} should be slightly modified: $\gamma^l$ turns into ${\Gamma_{t+l}}/{\Gamma_t}$.
\item
In the proof of \cref{thm:aiqi-eps-optimal}, the core inequality should be changed to
\[
\deltainfty (h_{<t})
<
\frac{\Gamma_{t+L}}{\Gamma_t}
+
\sum_{l=0}^{L-1} \frac{\Gamma_{t+l}}{\Gamma_t}
\big[
2 (2\beta \left( \frac{\tau}{|\mathcal A|} \right)^{-(l+1)} \hspace{-5pt}+ M^{-1} + \eta) + 2\tau
\big]
,
\]
so that for large enough $t$,
\[
\deltainfty (h_{<t})
<
\lambda^L
+
\sum_{l=0}^{L-1} \lambda^l
\big[
2 (2\beta \left( \frac{\tau}{|\mathcal A|} \right)^{-(l+1)} \hspace{-5pt}+ M^{-1} + \eta) + 2\tau
\big]
.
\]
\item
\cref{thm:aiqi-eps-optimal} thus generalizes to general discount sequences that decay faster than a geometric sequence. It suffices to change $\gamma$ with $\lambda$ in the criterion for parameters, and use $H=\limsup_t H_t(\eta) + 1$.
\item 
\cref{thm:aiqi-approaches-aixi} and \cref{thm:aiqi-multi-agent} generalize accordingly.
\end{itemize}

\section{AIQI-CTW}
\label{sec:experimental}

We introduce AIQI-CTW, a computable instantiation of AIQI with the context tree weighting \citep[CTW;][]{willems1995context} algorithm.
We also provide experiments that compare AIQI-CTW and MC-AIXI-CTW \citep{veness2011monte}.
The source code for experiments can be found at \url{https://github.com/yegonkim/aiqi}.

\paragraph{Context tree weighting.}
Context Tree Weighting (CTW) is a fast, principled method for doing Bayesian sequence prediction over binary strings. It builds a variable-order Markov model by maintaining a context tree up to some maximum depth. Rather than committing to a single context length, CTW mixes predictions from all context depths using an efficient recursive weighting scheme, which effectively performs Bayesian model averaging over a large family of context trees.
Remarkably, it makes online predictions by updating in time \emph{linear} in the context depth, assigns non-zero probability to all sequences, and tends to quickly exploit repeated structure.

\paragraph{Implementation.}
AIQI-CTW is an instantiation of AIQI that uses CTW as its return-predictor $\psi_n$.
MC-AIXI-CTW is similarly an approximation of AIXI that uses CTW as the environment model, and additionally a Monte Carlo Tree Search \citep[MCTS;][]{coulom2006efficient} algorithm that replaces exact planning.
AIQI-CTW uses $N$ CTW models, each corresponding to one of the $N$ return-predictors.
They receive the augmented sequences as described in \cref{sec:aiqi}, to predict the $M$-discretized $H$-step returns.
All rewards, observations, actions, and discretized returns are represented as blocks of bits.
Note that the weighted probabilities of CTW are updated only with the bits corresponding to returns.
This is in line with \cref{def:mixture-return-predictor}.
Expectation over returns is taken in an exact manner by computing all the probabilities.
The implementation of MC-AIXI-CTW was taken from the \texttt{pyaixi} repository\footnote{https://github.com/sgkasselau/pyaixi}, and AIQI-CTW was adapted from this implementation.

\paragraph{Experimental setup.}
We tested the algorithms on 3 environments: ``Biased Rock-Paper-Scissor'', ``Kuhn Poker'', and ``4$\times$4 Grid''. Biased Rock-Paper-Scissor is a repeated rock paper scissors game against an opponent with a simple exploitable bias: if it won the previous round by playing rock it plays {rock} again, otherwise it plays uniformly at random. Kuhn Poker is a simplified two-player, zero-sum poker game with a three-card deck (K, Q, J) in which players alternately choose between {pass} and {bet} under hidden information, capturing core phenomena such as bluffing and slow-playing. The 4$\times$4 Grid task is a small gridworld environment in which an agent moves on a 4-by-4 lattice using the primitive actions {up, down, left, right} (with boundary constraints) to reach a designated goal state and receive reward. More details on the environments and experimental setup can be found in \citet{veness2011monte}.

\paragraph{Hyperparameters.}
For MC-AIXI-CTW, we use the default hyperparameter used in \citet{veness2011monte}. One exception is the terminating age, which was reduced to keep the runtime manageable. The number of MCTS simulations in 4$\times$4 Grid was also reduced for the same reason.
For AIQI-CTW, we use a decaying $\eps$-greedy exploration and a minimum baseline exploration rate $\tau$. The CTW depth and horizon $H$ were chosen to equal those of MC-AIXI-CTW.
Note that we choose $N=H$, since the buffer period $N-H$ was introduced in \cref{sec:theoretical} solely as a theoretical device for proofs.
We report all the important hyperparameters in \cref{tab:hyperparams}.

\begin{table}[t]
\centering
\small
\begin{tabular}{lcccccc}
\toprule
\textbf{Parameter} &
\multicolumn{2}{c}{\textbf{Biased Rock-Paper-Scissors}} &
\multicolumn{2}{c}{\textbf{Kuhn Poker}} &
\multicolumn{2}{c}{\textbf{4$\times$4 Grid}} \\
\cmidrule(lr){2-3}\cmidrule(lr){4-5}\cmidrule(lr){6-7}
& \textbf{MC-AIXI-CTW} & \textbf{AIQI-CTW}
& \textbf{MC-AIXI-CTW} & \textbf{AIQI-CTW}
& \textbf{MC-AIXI-CTW} & \textbf{AIQI-CTW} \\
\midrule
Horizon $H$           & ---   & 4     & ---   & 2     & ---   & 12 \\
Period $N$      & ---   & 4     & ---   & 2     & ---   & 12 \\
Discretization $M$           & ---   & 9     & ---   & 9     & ---   & 13 \\
Baseline exploration $\tau$  & ---   & 0.01  & ---   & 0.01  & ---   & 0.01 \\

Exploration (initial)        & 0.999 & 0.999 & 0.99  & 0.999 & 0.999 & 0.999 \\
Explore decay rate           & 0.99999 & 0.9999 & 0.9999 & 0.9999 & 0.9999 & 0.9999 \\
CTW depth            & 32    & 32    & 42    & 42    & 96    & 96 \\
MCTS horizon                 & 4     & ---   & 2     & ---   & 12    & ---\\
MCTS simulations             & 200   & ---   & 200   & ---   & 40    & --- \\
Learning period              & 5000  & 100000& 5000  & 100000& 5000  & 100000 \\
Terminating age              & 10000 & 100000& 10000 & 100000& 10000 & 100000 \\
\bottomrule
\end{tabular}
\caption{Hyperparameters for MC-AIXI-CTW and AIQI-CTW across environments}
\label{tab:hyperparams}
\end{table}

\paragraph{Results.}
We performed each experiment on 8 seeds.
\cref{fig:experiment} is a plot of the exponential moving average (EMA) of reward (with $\alpha=10^{-3}$) against wall clock time. AIQI-CTW spends much less time in deciding an action than MC-AIXI-CTW, since it doesn't involve planning with MCTS.
In all experiments we find that AIQI-CTW has an advantage over MC-AIXI-CTW given the restricted computational budget.
Note that optimal performance can be obtained with MC-AIXI-CTW, albeit with orders of magnitude more compute, as described in \citet{veness2011monte}.

\begin{figure}[t]
\centering
\includegraphics[width=0.95\textwidth]{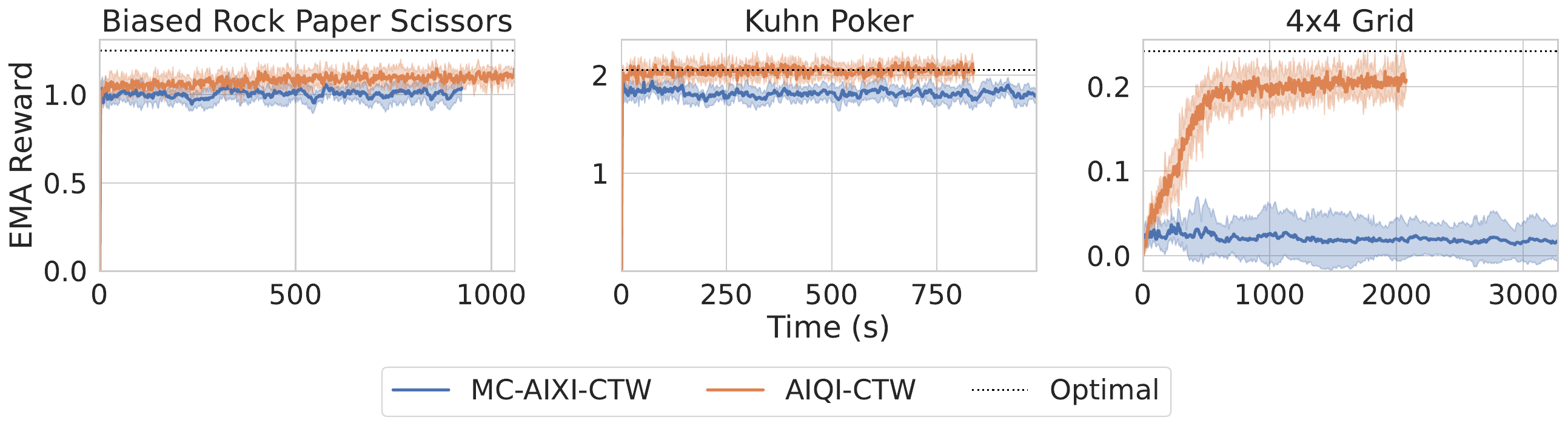}

\vspace{0.5em}

\includegraphics[width=0.95\textwidth]{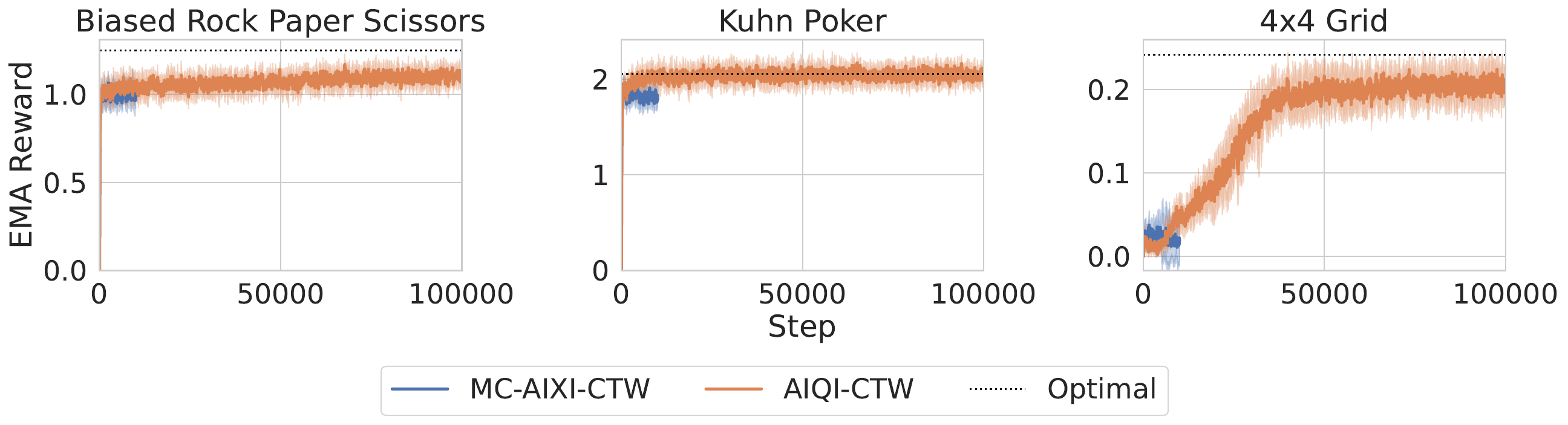}

\caption{
Plots of EMA reward vs wall clock time (in seconds) and environment steps on three environments.
}
\label{fig:experiment}
\end{figure}

\end{document}